\pdfoutput=1
\documentclass[accepted]{uai2021} %
\usepackage[british]{babel}

\usepackage{natbib} %
\usepackage{chapterbib}
\usepackage{mathtools} %
\usepackage{booktabs} %
\usepackage{tikz} %

\title{
Compositional Abstraction Error and a Category of Causal Models
}

\author[1]{\href{mailto:Eigil F.~Rischel <eigil.rischel@strath.ac.uk>?Subject=UAI 2021}{Eigil F.~Rischel}{}}
\author[2]{Sebastian Weichwald}
\affil[1]{%
    Department of Computer \& Information Sciences\\
    University of Strathclyde\\
    United Kingdom
}
\affil[2]{%
    Department of Mathematical Sciences\\
    University of Copenhagen\\
    Denmark
}

\usepackage{amsmath,amsthm,amsfonts,amssymb}
\usepackage{stmaryrd}
\usepackage{proof}
\usepackage{latexsym}
\usepackage{tikz,tikz-cd}
\usepackage{comment}

\usepackage{enumitem}
\setitemize{itemsep=0pt,topsep=6pt}
\usepackage{centernot}
\usepackage{xcolor}
\usepackage{calc}
\usepackage{geometry}
\usepackage{tikzit}

\input{comonoids.tikzdefs}

\tikzstyle{morphism}=[fill=white, draw=black, shape=rectangle]
\tikzstyle{medium box}=[fill=white, draw=black, shape=rectangle, minimum width=0.8cm, minimum height=0.9cm]
\tikzstyle{large morphism}=[fill=white, draw=black, shape=rectangle, minimum width=1.7cm, minimum height=1cm]
\tikzstyle{bn}=[fill=black, draw=black, shape=circle, inner sep=1.5pt]
\tikzstyle{state}=[fill=white, draw=black, regular polygon, regular polygon sides=3, minimum width=0.8cm, shape border rotate=180, inner sep=0pt]
\tikzstyle{medium state}=[fill=white, draw=black, regular polygon, regular polygon sides=3, minimum width=1.3cm, inner sep=0pt, shape border rotate=180]
\tikzstyle{large state}=[fill=white, draw=black, regular polygon, regular polygon sides=3, minimum width=2.2cm, shape border rotate=180, inner sep=0pt]
\tikzstyle{wn}=[fill=white, draw=black, shape=circle, inner sep=1.5pt]
\tikzstyle{large brace}=[draw, decorate, decoration={{brace,amplitude=12pt}}, inner sep=12pt, rotate=-90, xshift=34pt]

\tikzstyle{arrow}=[->]
\tikzstyle{dashed line}=[-, dashed]
\tikzstyle{double line}=[-, double]

\numberwithin{equation}{section}

\usepackage[capitalize,noabbrev]{cleveref}

\usepackage{minitoc}

\newtheorem{theorem}{Theorem}[section]
\newtheorem{proposition}[theorem]{Proposition}
\newtheorem{lemma}[theorem]{Lemma}

\newtheorem{definition}[theorem]{Definition}

\theoremstyle{definition}

\newtheorem{remark}[theorem]{Remark}

\newcommand{\tensor}{\otimes}

\DeclareMathOperator{\FinStoch}{FinStoch}
\DeclareMathOperator{\FinMod}{FinMod}
\DeclareMathOperator{\Err}{Err}
\DeclareMathOperator{\Met}{Met}

\DeclareMathOperator{\JSD}{JSD}

\newcommand{\doop}{\ensuremath{\operatorname{do}}}

\usepackage{pgfplots}

\renewcommand{\phi}{\varphi}
\newcommand{\spc}{\mathbf}
\DeclareMathOperator{\pa}{pa}
\DeclareMathOperator{\nd}{nd}
\DeclareMathOperator{\error}{error}

\definecolor{colA}{HTML}{3e6cc0}
\definecolor{colB}{HTML}{c2630a}
\definecolor{colC}{HTML}{660066}
\definecolor{colD}{HTML}{00bc55}

\usepackage{xr}

\usepackage{todonotes}

\begin{document}

\bibliographystyle{plainnat}

\maketitle

\begin{abstract}
Interventional causal models describe several
    joint distributions over some variables used to describe a system,
    one for each intervention setting.
They provide a formal recipe for how to move between the different joint distributions
    and make predictions about the variables upon intervening on the system.
Yet, it is difficult to formalise
    how we may change the underlying variables used to describe the system,
    say 
    moving from fine-grained to coarse-grained variables.
Here, we argue that compositionality is a desideratum
    for such model transformations and the associated errors:
    When abstracting a reference model $M$ iteratively,
    first obtaining $M'$ and then further simplifying that to obtain $M''$,
    we expect the composite transformation from $M$ to $M''$ to exist
    and its error to be bounded by the errors incurred by
    each individual transformation step.
Category theory, the study of mathematical objects via
    compositional transformations between them,
    offers a natural language
    to develop our framework for model transformations and abstractions.
We introduce a category of finite interventional causal models
    and, leveraging theory of enriched categories,
    prove the desired compositionality properties for our framework.
\end{abstract}

\section{Introduction}

With a causal model we aim to
    predict future observations of a system
    under the same conditions that held true when we devised
    the model (observational distribution),
    or when the system
    is subjected to external manipulating forces (interventional distributions),
    or infer what the observations
    would have been like had the context been different (counterfactuals).
Following the common aphorism ``all models are wrong'',
    we do not and cannot regard any causal model
    as a precise description of some system
    but instead as an approximate description that is
    pragmatically useful to us.
Besides the question of ``usefulness'',
    which inevitably depends on modelling goals,
    there is another one:
How well does our model approximate reality?

The ground-truth may be unattainable.
Therefore,
    an exact quantification and characterisation of how
    well a model describes reality may be beyond reach.
It is possible, however, to assess a model relative to another model.
Recent approaches aim to formalise this by considering the following question
    about transformations
    that link two models~\citep{rubenstein2017causal,beckers2019abstracting,beckers2019approximate}:
    How well does a transformed, simpler and higher-level causal model
    agree with or approximate another, more detailed and lower-level model?
An answer to this
    question is pragmatically useful.
It allows us
    to bound the error of a model used in practice
    relative to the corresponding most accurate 
    state-of-the-art reference model available
    (which substitutes ground-truth).
Based on the error relative to one reference model,
    we can make a principled decision between candidate
    causal models of varying aggregation level
    and degree of simplification.

Model development is an iterative process.
For example, specialists may propose approximations
    and transformations to only parts of a complex causal model,
    which we incorporate in a joint model
    and which we may refine further as we gain new
    insights.
It is desirable that an account of model transformation
    facilitate such modular step-by-step simplifications.
Existing work partly addresses this:
\citet[Lemma 5]{rubenstein2017causal} prove that the notion of exact transformations
is transitive: if $M$ can be exactly transformed into $M'$,
and $M'$ into $M''$, then the transformation from $M$ to $M''$ is also exact;
\citet[Theorem 3.9]{beckers2019abstracting} prove an analogous result
for a stricter notion of transformations;
and \citet[Section 5]{beckers2019approximate} discuss
the problem of
composing approximations and abstractions, instead of exact transformations,
in different order.

Complementing this line of research,
    we argue that a key desideratum
    for an account of model abstraction
    is
    that the measure of error be compositional.
We introduce a framework based on a category of interventional causal models,
    as category theory is a natural language for
    the study of models and compositional transformations between them.
Extending the result of 
\citet{rubenstein2017causal} to in-exact transformations,
    we prove that our proposed error notion is compositional:
    when we approximately transform $M$ into $M'$, and $M'$ into $M''$,
    then an approximate transformation $M$ to $M''$ exists
    and its error
    admits a natural bound in terms of the error of the two component transformations.

The categorical perspective is instructive and establishes
    an interdisciplinary common ground between applied category theory and
    causal modelling.
We prove the well-definedness of interventional distributions
    by string diagram surgery,
    thereby reproducing classical results in our category theoretical framework.
In the following subsections,
    we discuss
    conceptual issues in causal modelling
    that motivate our framework and can be addressed and formalised within it.
In \cref{sec:causalmodels}, we introduce our compositional framework,
    discuss design choices and its desirable features.
We discuss how 
    a notion of implemented models
    may enable us to reasonably talk about
    bounding the error of a causal model
    relative to ground-truth
    and based on interventional experiments in \cref{sec:functors}.
\cref{sec:discussion} concludes our article.

\subsection{Variables in statistical models}\label{sec:variables}

In probabilistic models,
the state of a system is represented by random variables.
The values are
determined by some measuring procedure.
For example, we may model the daily
warmth situation in an enclosed space
using the following variables:
$A_1$ and $A_2$ denoting the settings of two air conditioners (ACs) as transcribed to our spreadsheet
at 8 o'clock in the morning,
and $T$ denoting the reading of a room thermometer at noon.
Via a model over $(A_1,A_2,T)$,
we may
predict the AC configuration $(A_1,A_2)$ given $T=t$.

We often decide implicitly
which variables to use to describe a system.
This decision is constrained by measurement ability and accuracy,
computational and storage constraints,
and pragmatic considerations.
Yet, the choice is not arbitrary
and we do rely on our descriptors to capture relevant aspects of reality.
For example, to determine the value of $T$
we may choose to measure a proxy of
the average velocity of gaseous particles in a room with a thermometer,
reading it off at 8:01:32 every day, and rounding to 2 digits.
By doing so, we implicitly decide against
measuring all particles' momenta and including billions of variables,
and we fix which world states are (to be) considered indistinguishable for our modelling purposes.
    To make a conscious
    informed decision on the
    measurement protocols 
    and variables to use,
we need to understand how models using
    different variables relate,
    which is the subject of this article.

\subsection{Interventions in Causal models}\label{sec:interventions}

Interventional causal models
make predictions about $A_1,A_2$
not only for given/observed $T=t$ 
but also for situations when we manipulate
the value of $T$.
In our example,
we may write
$\doop(T=t)$
for an intervention
indicating that the value $t$ of $T$
did not come about naturally by following the measurement procedure outlined above,
but instead was set externally.
The intervention $\doop(T=t)$ may correspond, for example, to a situation where we
a)~entered the room shortly before noon
and lit a lighter next to the thermometer before
its reading was obtained, or
b)~taped the needle so the thermometer would have a fixed reading, or
c)~changed for a digital thermometer with a zero-digit precision display.

How a causal model relates to the tangible world
    thus depends not only on the measurement procedures
    to obtain the variable values
    but also on the physical implementations of interventions.
Both 
    determine how our model variables,
    which serve as explanans in our causal description,
    relate
    to reality or a state-of-the-art reference model
    instead of
    the unattainable ground-truth.
There are multiple measurement procedures
    to obtain the variables' values.
Additionally, the 
    physical procedure corresponding
    to an intervention $\doop(T=t)$ in our model is ambiguous.
For our causal model transformations,
    we thus require the modeller to make the
    content of a causal model that is to be preserved explicit
    and to specify
    how observations and interventions map between
    the transformed model and the reference model.

\subsection{Ambiguous interventions}\label{sec:ambiguousmanipulations}

It is problematic and unrealistic
    if—as is customary—the
    predictions a model makes about the effects of interventions
    do not depend on how the intervention is being implemented~\citep{spirtes2004causal}.
Even if interventional data were available,
    the definition of causal variables may be underdetermined~\citep{eberhardt2016green}.

Therefore, we require the modeller to make explicit
    how high-level interventions are implemented on the low-level.
Ambiguity in the low-level implementation of interventions
    is encoded in an intervention kernel.
For example,
we can encode if we only ever intervened on total cholesterol (TC) by prescribing a certain restricted set of diets
with 
comparable effects on the level of low-density (LDL) 
and high-density lipoprotein (HDL).
Whether we allow only a restricted set of diets or any diet
carries over, via the intervention kernel,
to a lower or higher abstraction error when transforming a model
and replacing LDL and HDL by TC.
\subsection{Why compositionality?}\label{sec:motivation}

Without compositionality,
    we cannot assess a model's error by comparing against the predecessor model,
    but instead need to evaluate each model relative to the reference model.
The following example illustrates how this lack of modularity 
    deceives us when iteratively simplifying and
    abstracting different parts of a model.
To compare the interventional distributions implied by the transformed model
    to those of the reference model,
    we here use the KL-divergence
    as a common measure of how
    one distribution differs from a reference distribution.
    We motivate why an information theoretic measure should be used in \cref{sec:assessing}.

Imagine a model $M$ of two ACs affecting the temperature in a room. 
The variables are $A_{1},A_{2}$ for the ACs settings,
    and $T$ for the temperature.
The causal structure is \(A_1 \to T \leftarrow A_2\), and we define the mechanisms by stating the respective kernels.
Assume both ACs have settings \(\{1,\dots, 100\}\)
and 
the temperature scale is
\(\{1,\dots, 100\}\).
Suppose that
\begin{itemize}
\item \(P(T=n \mid A_1=100, A_2=100) = Z / n^2\),
\item \(P(T=n \mid A_1=100, A_2\neq 100) = Z' / n^3\),
\item \(P(T=n \mid A_1\neq 100) = Z'' / 100^n\),
\item the settings of the ACs are independently and uniformly distributed
in the observational setting, and
\item $Z,Z',Z''$ are normalising constants.
\end{itemize}

Consider an abstraction $M'$ of $M$,
where the influence of the AC $A_2$ is
simplified away, that is,
we remove the arrow \(T\gets A_2\).
Define the simplified model such that the distributions,
    $P(T \mid A_{1}), P(A_{1}), P(A_{2})$
    agree with the reference model $M$ as closely as possible.
As the simplified model requires that $T$ and $A_{2}$ be independent,
    the two models disagree on $P(T \mid A_{2})$.
We compare the two models' predicted distribution for
    $T\mid\doop(A_1=100,A_2=100)$.
A calculation shows that the KL-divergence between the two models' interventional distributions
    is a fairly small \(0.22\).

Consider another abstraction \(M''\) of $M'$
where the influence of the AC $A_1$ is also
simplified
away.
The KL-divergence between the predictions of \(M'\) and \(M''\)
on $T\mid\doop(A_1=100,A_2=100)$
is
again
a fairly small \(0.39\).

We may expect that
    abstracting $M$ as $M''$ should be permissible,
    since both abstractions, from $M$ to $M'$ and from $M'$ to $M''$ were permissible.
While small errors may accumulate,
    we expect
    the error of the abstraction $M$ to $M''$
    to be bounded in terms of the errors of abstracting $M$ as $M'$ and $M'$ as $M''$.
One way to formalise this expectation is
    to impose the triangle inequality
    such that the transformation error of $M$ to $M''$ is bounded by \(0.39 + 0.22 = 0.61\).
The KL-divergence between the predictions of \(M\) and \(M''\), however,
    is more than twice that: \(1.52\).
For $\doop(A_1=100,A_2=100)$,
$M$ predicts a temperature above $1$ in $~39\%$ of cases,
while $M''$ predicts a temperature above $1$ only in $~0.01\%$
of the cases.
The errors of the individual abstraction steps
    are not indicative
    of how well the abstracted model $M''$
    approximates the reference model $M$.
The discrepancy between the individual and overall abstraction error is unbounded;
    in \cref{appendix_kl_counterexample}.
    we construct
    for any \(\epsilon, K > 0\),
    a situation
    where the individual abstractions incur KL-divergences \(\leq \epsilon\),
    while the overall abstraction incurs a KL-divergence \(> K\).

This is a severe hurdle to the development of causal models.
We often do not have access to ground truth but only have a model we think is reasonably accurate
    and are considering replacing it with another, perhaps simpler, one.
If the errors are not compositional, as in the above example,
    we cannot ensure that
    the abstracted model $M''$ closely approximates the reference model (or ground-truth) $M$
    by enforcing a small approximation error relative to a previously
    established good approximation $M'$.
This limits how informative the efforts to develop $M'$ and to empirically validate
its close resemblance of $M$ are about $M''$
(see also~\cref{sec:functors}).

Therefore, compositionality is a desideratum
    for our notion of causal model abstraction.
Our framework is compositional:
    the errors of the two abstractions above
    are $0.2$ and $0.22$,
    and the composite abstraction error is $0.37 \leq 0.42$.

    While abstraction examples similar to those discussed by \citeauthor{rubenstein2017causal,beckers2019abstracting,beckers2019approximate}
    can be expressed within our framework,
    the above example -- where the ``simplification'' consists of
    deleting causal arrows, rather than, for example, reducing the number of variables --
    is instructive to
    exhibit the failure of compositionality.

\subsection{assessing model abstraction error}\label{sec:assessing}

Often, we cannot establish an exact correspondence
    between two modelling levels.
The conditions
    for 
    the transformations and abstractions discussed
    by \citet{rubenstein2017causal} and \citet{beckers2019abstracting},
    and for modelling equilibria of a time-evolving
    process~\citep{janzing2018structural} are restrictive.
Therefore, we also wish
to characterise
    transformations in which the
    transformed model only approximates the
    causal relationships in the reference model.
Still, 
    we require the two models,
    where one can be viewed as a transformation of the other,
    to stand in a well-defined relationship.
A natural idea is to ask for an {approximate} transformation
    that preserves the causal structure \emph{up to some error}.
This idea has been proposed, for instance, by \citet{beckers2019approximate}.

In 
\citeauthor{beckers2019approximate}'s approach,
    the measure of
    abstraction error
    ultimately depends on a choice of metric on the underlying set of outcomes.
The advantage of such an approach is obvious:
it allows us to adjust the metric to optimally capture those aspects
of the model that are of interest.

We see two problems with this approach.
First, choosing a metric on the set of {high-level} variables
    requires that we already have chosen that abstraction level.
In order to assess an abstraction we are required
    to decide how important the high-level variables we have just invented are.
This
    is an unnatural requirement
    for comparing and finding candidate abstractions.
Second,
    having to choose a metric at all
    requires detailed knowledge about the system
    and the model's intended use.
If we know which task we are solving with a model, however,
    we can instead assess the abstracted model directly and evaluate
    the actions it recommends for this task~\citep{kinney2020causal}.
We need not require that the model
    approximates some reference model as long as
    it is useful to solve the given task.
If, instead, we do not know what the abstracted model is to be deployed for,
    we need to revert to an error measure
    that rates models
    by their resemblance of a reference model
    such that we 
    can select models that are useful for a wide range of tasks.
Our information theoretic error measure evades the arbitrary choice of a metric
    on the outcome space.

\subsection{A categorical approach
and compositional transformations}
    
We propose \emph{compositionality} as a desideratum for notions of error
    for transformations between causal models
        (cf.\ \cref{sec:motivation}).
In our framework,
    the error of a composite transformation 
    can be bounded in terms of the error of the component transformations:
        if $f: M \to M'$ and $g: M' \to M''$ are transformations between causal models
        and $e,e'$ are the corresponding errors,
        then the composite transformation
        $g \circ f: M \to M''$ exists and has error at most $e + e'$.

We propose the language of \emph{category theory}
    to discuss abstractions
    with a focus on compositionality.
Category theory 
    studies mathematical objects in terms of the compositional transformations between them.
A \emph{category} $\mathcal{C}$ consists of
\begin{enumerate}
\item a collection of \emph{objects}, usually written $\operatorname{ob} \mathcal{C}$,
\item for each pair of objects $X,Y$, a collection of \emph{morphisms} $f: X \to Y$,
\item a notion of associative and unital \emph{composition},
    assigning to each pair of morphisms $f: X \to Y$, $g: Y \to Z$
    a composite $g \circ f: X \to Z$, and
\item for each object $X$, an identity morphism $1_{X}: X \to X$.
\end{enumerate}

For example,
    there is a category where the objects are vector spaces,
    the morphisms are linear functions,
    and composition is composition of functions.

    A category is
    an abstract
    mathematical structure,
    where ``object'' and ``morphism''
    label two parts of that structure.
    For example,
    we can
    define a category
    with two objects $\{\bigstar,\clubsuit\}$,
    the identity morphisms $1_\bigstar: \bigstar \to \bigstar$ and $1_\clubsuit: \clubsuit\to \clubsuit$,
    and one morphism $f: \bigstar \to \clubsuit$.
    The notation $f: X \to Y$
    is overloaded and implies that $f$ is some morphism between the objects $X$
    and $Y$,
    rather than a \emph{function} $f$ from $X$ to $Y$.

To establish an interdisciplinary meeting ground,
    we introduce the categorical concepts
    on a level of detail
    necessary to understand and gain intuition about our framework.
See, for example,
    \citet{maclane} or \citet{riehl2017category},
    for comprehensive introductions to category theory.

\section{Causal Models} \label{sec:causalmodels}

To ease the exposition of
our categorical framework, we
briefly recap the conventional introduction
of causal models and their key properties
in \cref{sec:conventional}.
In \cref{sec:category},
we introduce a category of finite
interventional causal models,
where the objects are
causal models over finitely many variables with finite outcome spaces.
The morphisms in this category
are model transformations with an associated error
which we formalise via an enriched category
and prove to be compositional in 
\cref{sec:compositional}.

\subsection{Recap: Structural Causal Models}\label{sec:conventional}

For context,
we briefly introduce
Structural Causal Models (SCMs).
For details, we refer to, among others,
\citet{Spirtes2000,Pearl2009,bollen1989structural,peters2017elements}.

\begin{definition}[Structural Causal Model (SCM)]
Let $\mathbb{I}$ be an index set.
Let $E = (E_i:i\in\mathbb{I})$ be a collection of independent variables
with distribution $\mathbb{P}_E = \otimes_{i\in\mathbb{I}} \mathbb{P}_{E_i}$.
Let $\mathbf{S}$ be a set of structural equations
$$X_i = f_i(X_1,...,X_{i-1},E_i)$$
for $f_i: \prod_{j=1}^{i-1}\mathbf{X}_j \times \mathbf{E}_i \to \mathbf{X}_i$
measurable functions and $i\in\mathbb{I}$.
Let $\mathbf{I}$ be a set of interventions
which we denote by $\doop(X_{k_1} = x_{k_1}, ..., X_{k_l} = x_{k_l})$
for $k_1, ..., k_l \subseteq [d]$ and $x_{k_j} \in \mathbf{X}_{k_j}$
and which identify some structural equations in $\mathbf{S}$
to be replaced by the equations $\mathbf{X}_{k_j} = x_{k_j},\ j\in[l]$.

We call $(\mathbf{S},\mathbf{I},\mathbb{P}_E)$ a structural causal model.
\end{definition}

SCMs induce sets of distributions over $X = (X_i : i\in\mathbb{I})$.
The distribution $\mathbb{P}_X^\varnothing$
induced by $\mathbb{P}_E$ and the structural equations $\mathbf{S}$
is called the observational distribution.
Further distributions $\mathbb{P}_X^{\doop(i)}$ are obtained for each intervention $i \in \mathbf{I}$
by changing the respective structural equations according to $i$ and considering
the distribution induced by $\mathbb{P}_E$ and this new set of structural equations.
Thus, an SCM can be understood as a structured set of joint distributions
over $X$ where the distributions are indexed by
interventions.
The mental picture may be depicted as\\[-1.8em]
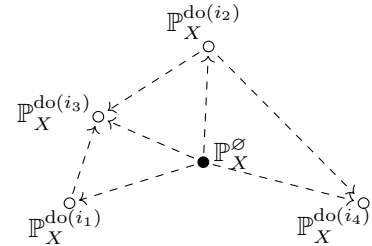
\centerline{\begin{tikzpicture}

\begin{axis}[hide axis,scale=.8]
\addplot[white]
table {%
34.002525	1.3731846
36.389637	1.4206893
39.168667	1.7745008
42.134747	2.7290497
45.092907	4.081451
48.238117	6.0345893
50.796448	7.478536
53.941658	9.431675
57.695522	10.799911
60.874386	11.062177
63.858276	11.121558
67.24794	10.791004
71.39471	10.415815
75.276146	10.373654
78.67897	10.381671
82.07467	10.747755
84.157455	11.08771
86.5327	11.731994
89.30787	13.279758
91.07445	14.5089445
92.53908	15.911225
93.94285	17.371998
95.390854	19.609768
96.245636	21.65663
97.20078	24.660711
97.69276	26.93916
98.01163	28.915655
98.267265	31.069996
98.04785	33.095478
97.58973	35.116215
96.64706	37.48551
95.65896	39.137493
94.07289	40.837273
92.00584	42.706585
89.7645	44.333622
87.26665	46.851078
85.06836	49.314793
83.526535	51.791573
81.404854	56.406067
80.01475	60.259
78.82386	63.10097
77.65198	64.9881
76.05404	67.28466
74.35336	68.74335
72.53097	70.31902
70.41968	71.41134
68.00762	72.617065
65.606255	73.2857
63.206078	73.894646
60.992058	74.14909
58.843655	74.10634
56.57827	73.94186
54.377316	73.53985
52.473564	73.20345
50.93263	72.635475
49.335575	71.887276
47.912807	71.38135
46.615334	70.57941
44.960983	69.71066
42.953312	68.596085
40.88715	67.42064
38.237278	65.57686
35.88698	63.679344
33.895935	61.72927
31.846401	59.718338
29.271643	57.099934
26.286266	54.11515
23.650642	51.555237
21.768267	50.14464
20.58896	49.404755
19.17213	48.60044
18.058441	47.56335
17.181087	46.65037
16.125887	45.67415
14.294876	44.682484
12.330257	44.404583
10.292896	44.781948
9.023035	45.5925
6.739837	46.32318
4.760967	46.761414
2.7319198	46.721035
1.5490512	46.16018
0.49741444	45.004925
-0.17715319	42.901947
-0.31818295	40.988693
-0.033154365	38.666008
0.6313185	35.276245
1.6992849	32.61093
4.3939924	29.20187
6.462231	27.27288
9.201178	24.64082
11.886384	21.70918
13.677611	18.700052
14.830195	14.78269
15.67489	11.336814
17.105673	8.439914
18.634436	6.6195908
20.757603	4.9304996
23.71507	3.3177128
27.201326	2.1333597
30.735384	1.5469728
32.70594	1.5264864
34.083282	1.3150902
};

\draw[black, fill=black] (axis cs:60.862015,26.68382) circle (2pt) node (0){};
\node  at (axis cs:69.00893,28.801174) {$\mathbb{P}^{\varnothing}_X$};

\draw (axis cs:22.86147,16.211173) circle (2pt) node (1){};
\node at (axis cs:21.846848,10.195957) {$\mathbb{P}^{\doop(i_1)}_X$};

\draw (axis cs:31.034046,38.289234) circle (2pt) node (2) {};
\node at (axis cs:18.57733,39.4891) {$\mathbb{P}^{\doop(i_3)}_X$};

\draw (axis cs:62.363655,56.226143) circle (2pt) node (3) {};
\node at (axis cs:61.132545,62.589703) {$\mathbb{P}^{\doop(i_2)}_X$};

\draw (axis cs:106.363655,16.226143) circle (2pt) node (4) {};
\node at (axis cs:98.132545,10.589703) {$\mathbb{P}^{\doop(i_4)}_X$};

\draw[->,dashed] (0) to (1);
\draw[->,dashed] (0) to (2);
\draw[->,dashed] (0) to (3);
\draw[->,dashed] (0) to (4);
\draw[->,dashed] (1) to (2);
\draw[->,dashed] (3) to (2);
\draw[->,dashed] (3) to (4);

\end{axis}

\end{tikzpicture}}\\[-1.5em]
The distributions model the system under different interventions that force
    certain variables to take on certain values.
The partial ordering of the distributions
    reflects the compositionality of interventions,
    which is crucial for
    causally consistent reasoning between two models~\citep{rubenstein2017causal}.
Acyclicity is a common assumption
and
    ensures that observational and interventional distributions are well-defined;
    cyclic models are intricate~\citep{bongers2018theoretical}.

\subsection{A category of finite interventional causal models}
\label{sec:category}

We begin by defining the objects of our category
$\FinMod$ of finite interventional causal models:

\begin{definition}[Finite interventional causal model]
  A \emph{finite interventional model} $M=(G^{M},\mathbf{X}^{M},\phi)$ consists of
  \begin{enumerate}
    \item a finite directed acyclic graph $G^{M}$,
    called causal graph of $M$, with vertices called \emph{variables} $V(M)$ of $M$,
    \item for each variable $v$, a finite set $\mathbf{X}^{M}_{v}$ of possible values of $v$, and
    \item for each variable $v$, a Markov kernel called \emph{mechanism}
          \[\phi^{M}_{v}: \prod_{v': v' \to v\text{ in }G^M} \mathbf{X}_{v'}^{M} \to \mathbf{X}^{M}_{v}.\footnotemark[1]\]
  \end{enumerate}
  For each root node $v$, there is a kernel $\phi_{v}^{M}: * \to \mathbf{X}^{M}_{v}$,
  that is, a probability distribution on $\mathbf{X}^{M}_{v}$.
\end{definition}

\footnotetext[1]{
In other words, a function $\prod_{v' \to v}\mathbf{X}^M_{v'} \to \Delta(\mathbf{X}^{M}_{v})$.\\
Since we do not consider counterfactuals in the present article,
it is sufficient to
specify these kernels instead of
functional equations and distributions on the exogenous variables.
}

A finite interventional causal model
    induces distributions:
    
\begin{definition}[Interventional distributions]\label{def:interventionaldists}
  Given a subset $S\subseteq V(M)$ of the variables in a model $M$
  and corresponding values $x_{v},v\in S$,
  there is a well-defined \emph{interventional distribution}, a kernel
  \[I_{S}: \prod_{v \in S} \mathbf{X}^{M}_{v} \to \prod_{v \in V({M})} \mathbf{X}^{M}_{v},\] 
  determined by the condition 
  that for $v \in S$, $I_{S}(x)_{v} = x_{v}$ with probability $1$,
  and the conditional distribution of each \emph{other} variable $v\centernot\in S$,
  given its parents,
  is given by the mechanism $\phi_{v}^M$.
  
  When $S$ is empty,
    $\prod_{v\in S}\mathbf{X}^{M} = *$,
    we obtain the \emph{observational distribution}
    as the joint distribution over all variables
    under the null intervention.

  We may write $P(- \mid \doop(v = x_{v}, v \in S))$
  for
  $I_{S}((x_{v})_{v \in S})$.
\end{definition}

Acyclicity of the causal graph ensures that the distributions
    in \cref{def:interventionaldists} are well-defined.
To establish the parallels between the categorical and the classical perspective,
    we (re-)prove this result using string diagram surgery
    in \cref{sec:diagrammatic_proofs}.
String diagrams are
widely used in category theory to depict constructions in \emph{monoidal categories}
such as $\FinStoch$~\citep{fritz2019synthetic}.
The proofs rely on rewiring diagrams such as the following
(read bottom-to-top)\\[.5em]
\centerline{\scalebox{0.7}
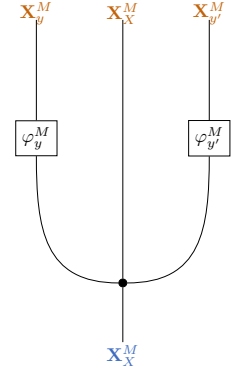
{\begin{tikzpicture}
	\begin{pgfonlayer}{nodelayer}
		\node [style=none] (48) at (12.5, -10.5) {{\color{colA}$\mathbf{X}^M_X$}};
		\node [style=none] (50) at (15.75, 2.5) {{\color{colB}$\mathbf{X}^M_{y'}$}};
		\node [style=bn] (51) at (12.5, -7.75) {};
		\node [style=none] (54) at (9.25, 2.5) {{\color{colB}$\mathbf{X}^M_y$}};
		\node [style=none] (55) at (12.5, 2.5) {{\color{colB}$\mathbf{X}^M_X$}};
		\node [style=morphism] (56) at (9.25, -2.25) {{$\phi_y^M$}};
		\node [style=none] (58) at (12.5, -10) {};
		\node [style=none] (59) at (12.5, 2.25) {};
		\node [style=none] (60) at (9.25, 2.25) {};
		\node [style=none] (61) at (15.75, 2.25) {};
		\node [style=morphism] (62) at (15.75, -2.25) {{$\phi_{y'}^M$}};
	\end{pgfonlayer}
	\begin{pgfonlayer}{edgelayer}
		\draw (58.center) to (51);
		\draw [in=-90, out=180, looseness=1.25] (51) to (56);
		\draw (56) to (60.center);
		\draw (51) to (59.center);
		\draw (62) to (61.center);
		\draw [in=-90, out=0, looseness=1.25] (51) to (62);
	\end{pgfonlayer}
\end{tikzpicture}}}\\[.5em]
without changing the resulting distribution;
above diagram, for example, depicts a kernel
\({\color{colA}\mathbf{X}_X^{M}} \to  {\color{colB}\mathbf{X}_y^{M}} \times {\color{colB}\mathbf{X}_X^{M}} \times
{\color{colB}\mathbf{X}_{y'}^{M}}\),
informally described as ``given \(x \in \mathbf{X}_X^M\),
sample \(y \in \mathbf{X}_y\) from the distribution
\(\phi^M_y(x)\)
and independently sample \(y'\) from the distribution
\(\phi^M_{y'}(x)\), then return the tuple \((y,x,y')\)''.

Next, we define
    the morphisms in $\FinMod$:

\begin{definition}[Model transformation]
  A \emph{transformation of models} $f: M \to M'$ consists of
  \begin{enumerate}
    \item a surjective \emph{vertex} map $f_{V}: V({M}) \to V({M'})$,
    \item for each $v \in G^{M'}$, a \emph{measurement} function
    \[f^{m}_v: \prod_{f_{V}(v') = v}\mathbf{X}^{M}_{v'} \to \mathbf{X}^{M'}_{v}\text{, and}\]
    \item for each $v \in G^{M'}$, an \emph{intervention} kernel
    \[f^{i}_{v}: \mathbf{X}^{M'}_{v} \to \prod_{f_{V}(v')=v} \mathbf{X}^{M}_{v'}.\]
  \end{enumerate}

\end{definition}

A model transformation is interpreted as follows:
  \begin{enumerate}
    \item Each high-level\footnote{When $f: M \to M'$,
    we may, for enhanced intuition, think of the models $M$ and $M'$ as ``low-level'' and ``high-level'', respectively.
    It is not required that $M'$ be more high-level than $M$.} variable $v \in V({M'})$ in $M'$ abstracts over
    a set
    $f_{V}^{-1}(\{v\})$
    of low-level variables in $M$.
    
    \item The high-level observation of $v$
    is determined by the 
    values $x_{v'}$ of the low-level variables $v'$
    via $f^m_v(x_{v'})$.
    
    \item For each intervention $\doop(v=x_v)$ on high-level variables $v$,
    there is a \emph{distribution} $f^i_v(x_v)$
    of corresponding interventions on the low-level variables $f_{V}^{-1}(\{v\})$.
  \end{enumerate}
  To sum up, $f^{m}_{V}$ is a map \emph{from low- to high-level} \textbf{m}easurements,
  while $f^{i}_{V}$ is a map \emph{from high- to low-level} \textbf{i}nterventions.

This notion of model transformation satisfies the desiderata:
\begin{itemize}
\item Variables in the high-level model are explicitly defined relative
    to the reference model (cf.~\cref{sec:variables}).
    
\item The relation of high-level interventions 
    to interventions in the 
    reference model
    is explicit
    (cf.~\cref{sec:interventions}).

\item
    The intervention kernel explicitly
    captures the uncertainty in the low-level implementation
    of high-level interventions,
    enhancing transparency and allowing for the error due
    to intervention ambiguity to be part of
    our error when assessing the abstraction:
    if a high-level intervention is related to
    low-level interventions with substantially
    different effects we expect a large approximation error
    (cf.~\cref{sec:ambiguousmanipulations}).
\end{itemize}

The last point bears additional explanation.
\citet{rubenstein2017causal} argue for restricting the intervention set
    on the low-level to enable any simplification at all
    when moving to a higher-level model.
While this approach is also transparent about which content of
    a causal model is to be abstracted away,
    the generalisation proposed by \citet{beckers2019approximate}
    offers a practical advantage.
While ``setting the temperature to $t$'' in a high-level model
    may be implemented by multiple configurations
    of all gaseous particles' velocities,
    not all low-level implementations are possible or equiprobable.
It may be impossible, for instance,
    that the temperature be raised by imparting an absurdly high velocity
    to a single molecule while leaving the others unchanged.
Following \citet{rubenstein2017causal},
    we may remove such interventions from the set
    of valid interventions.
Yet, their approach cannot encode
    that among all possible low-level configurations
    with the same average velocity,
    some are more probable than others;
    especially if only certain actionable interventions,
    such as setting several ACs in a room,
    are considered on the higher-level.
We therefore follow \citet{beckers2019approximate} and
    demand high-level interventions to be linked to low-level interventions
    by an intervention kernel instead of a one-to-many mapping.

\subsection{Compositional error}\label{sec:compositional}

We have not yet imposed any compatibility between the distributions induced
    by the high-level and the low-level causal model.
We develop the notion of transformation error
    to reflect the level of agreement between two models
    with a morphism between them.
The notion depends on the Jensen-Shannon divergence and the
category $\FinStoch$ of kernels between finite sets.
All proofs can be found in \cref{sec:proofs}.

\begin{definition}[$\FinStoch$ (see also \citet{fritz2019synthetic})]
  Let $\FinStoch$ be the category where
  \begin{enumerate}
    \item objects are finite sets,
    \item a morphism $f$ from $\mathbf{X}$ to $\mathbf{Y}$ is a kernel,
    that is, a map from $\mathbf{X}$ to the set $\Delta(\mathbf{Y})$ of probability distributions on $\mathbf{Y}$,\footnote{One can think of a morphism $f: \mathbf{X}\to\mathbf{Y}$ in $\FinStoch$
    as a stochastic matrix where
    entries reflect the probabilities to transition from $x\in\mathbf{X}$ to $y\in\mathbf{Y}$.},
    and
    \item composition is by integration: if $f: \mathbf{X} \to \mathbf{Y}$ and $g: \mathbf{Y} \to \mathbf{Z}$ are kernels, their composition $g\circ f$ is
          \[(g \circ f)(x) = \int_{y\in\mathbf{Y}} g(y) \mathop{}\!\mathrm{d}(f(x))(y) \in \Delta(\mathbf{Z}).\]
  \end{enumerate}
  $\FinStoch(\mathbf{X},\mathbf{Y})$ denotes the set of kernels from $\mathbf{X}$ to $\mathbf{Y}$.
\end{definition}

To measure the error introduced by a causal model transformation,
    we define a distance between probability measures
    based on the \emph{Jensen-Shannon divergence} (JSD)~\citep{JSDES}:
\begin{definition}[Jensen-Shannon divergence (JSD)]
  Let $p_{0},p_{1}$ be distributions on a finite set $\mathbf{X}$.
  Let $B$ be a random variable with $P(B=0) = P(B=1) = 1/2$, that is, a fair coin flip.
  Let $X$ be a random variable, valued in $\mathbf{X}$, with $P(X=x \mid B=i) =p_{i}(x)$.
  The Jensen-Shannon divergence $\JSD(p_{0},p_{1})$ is defined as the mutual information between $B$ and $X$.
\end{definition}

Intuitively, JSD answers the following question:
If we learn the value $x$ of $X$, 
    sampled either from $p_{0}$ or $p_{1}$,
    how much information does $X=x$ reveal
    about which of the two distributions it was sampled from?
Based on the JSD,
    we define a distance between probability measures
    as follows:
    
\begin{definition}[Jensen-Shannon distance]
  For $f,g\in\FinStoch(\mathbf{X},\mathbf{Y})$,
  the \emph{Jensen-Shannon distance}
  is defined as
  \[d_{\JSD}(f,g) = \sup_{x\in \mathbf{X}} \sqrt{\JSD(f(x),g(x))} \leq 1.\]
$d_{\JSD}$ is a metric on $\FinStoch(\mathbf{X},\mathbf{Y})$.
\end{definition}

\begin{proposition}[kernel composition is short]\label{prop:short}
  The composition of kernels
  \[\FinStoch(\mathbf{X},\mathbf{Y}) \otimes \FinStoch(\mathbf{Y},\mathbf{Z}) \to \FinStoch(\mathbf{X},\mathbf{Z})\]
   is a short map, that is,
  for any $f_1,f_2\in\FinStoch(\mathbf{Y},\mathbf{Z})$,
  $g_1,g_2\in\FinStoch(\mathbf{X},\mathbf{Y})$
  it holds that
  \[d_{\JSD}(f_{1}\circ g_{1},f_{2}\circ g_{2}) \leq d_{\JSD}(f_{1},f_{2}) + d_{\JSD}(g_{1},g_{2}).\]
\end{proposition}

\begin{remark}
The above can be summarized as follows:
$d_{\JSD}$ defines an \emph{enrichment} of $\FinStoch$
in the monoidal category $\Met$ of metric spaces.
We provide a brief description of enriched categories in \cref{sec:enriched}.
\end{remark}

In the following lemma we prove that JSD is compositional,
    which is key for the compositionality of our notion of abstraction error.
JSD is only an exemplary choice of distance.
More precisely, 
    causal model transformation error can be defined analogously
    and its compositionality is guaranteed by \cref{prop:error_composes}
    also for any other distance
    for which we can prove analogs of \cref{prop:short,lemma:jsdcomposes}.

We use
diagrams to depict some collection of objects and morphisms in a category.
We encourage readers
    unfamiliar with these diagrams
    to understand
    ``consider a diagram ...'' as
    ``consider a collection of finite sets and kernels ...''.

\begin{lemma}[JSD is compositional]\label{lemma:jsdcomposes}
  Consider a diagram (not necessarily commutative\footnote{
  A diagram of this type is \emph{commutative} 
  if each way of composing the depicted morphisms yields the same result –
  in this case, if $f = {\color{colA}b'ga}$ and $g={\color{colB}c'hb}$.})
  in $\FinStoch$ of the following form:
  \begin{equation}\label{commdiag}
  \begin{tikzcd}
    A \ar[colA,d, "a"] \ar[r, "f"] & A' \\
    B \ar[colA,r, "g"] \ar[colB,d, "b"] & B' \ar[colA,u, "b'"]\\
    C \ar[colB,r, "h"] & C' \ar[colB,u, "c'"]
  \end{tikzcd}
  \end{equation}
  Then $d_{\JSD}(f,{\color{colA}b'}{\color{colB}c'hb}{\color{colA}a})
  \leq d_{\JSD}(f, {\color{colA}b'ga}) + d_{\JSD}(g,{\color{colB}c'hb})$.
\end{lemma}

Visually, we imagine first replacing the morphism $f$ with {\color{colA}$b'ga$ },
incurring error 
$d_{\JSD}(f, {\color{colA}b'ga})$,
then replacing the morphism $g$ with {\color{colB}$c'hb$},
incurring error
$d_{\JSD}(g,{\color{colB}c'hb})$.
\cref{prop:short} ensures that the alteration of one part of a composition
does not create more error than the error associated with the alteration itself.
The triangle inequality ensures that successive alterations combine in a natural way.

We define our notion of compositional transformation error:

\begin{definition}[Transformation error]
    Let $f: M \to M'$ be a transformation of models in $\FinMod$,
    and $S \subseteq V(M')$ a subset of variables.
    
    The error associated with $S$
    is
    the Jensen-Shannon distance
    $d_{\JSD}\left({\color{colA}I^{M'}_{\{v_{i}\}}}\ ,\ 
    {\color{colD}f^{m}}\circ {\color{colC}I^{M}_{f_{V}^{-1}(\{v_{i}\})}} \circ {\color{colB}f^{i}}\right)$,
    which reflects the failure
    of the following diagram
    to commute
    \begin{center}
      \begin{tikzcd}
        {\prod_{i\in S} \mathbf{X}^{M'}_{v_i}} \ar[colB,thick,d] \ar[colA,thick,r] & {\prod_{v \in V(M')}\mathbf{X}^{M'}_v}
        \ {\footnotesize\rotatebox[origin=c]{15}{\textcolor{gray}{\textit{high}}}}
        \\
        {\prod_{i \in S}\prod_{f(v) = v_i}\mathbf{X}^{M}_v} \ar[colC,thick,r] & {\prod_{v \in V(M)} \mathbf{X}^M_v}
        \ {\footnotesize\rotatebox[origin=c]{15}{\textcolor{gray}{\textit{low}}}}
        \ar[colD,thick,u] \\[-2.2em]
        \ {\footnotesize\rotatebox[origin=c]{15}{\textcolor{gray}{\textit{intervention}}}}
        &
        \ {\footnotesize\rotatebox[origin=c]{15}{\textcolor{gray}{\textit{distribution}}}}
      \end{tikzcd}
    \end{center}

    The error of $f$ is the maximal error associated with any subset of $V(M')$:
    \vspace{-.5em}
    \[
     \error(f) = \max_{S\subseteq V(M')}
    d_{\JSD}\left({\color{colA}I^{M'}_{\{v_{i}\}}}\ ,\ 
    {\color{colD}f^{m}}\circ {\color{colC}I^{M}_{f_{V}^{-1}(\{v_{i}\})}} \circ {\color{colB}f^{i}}\right)
    \]
    \vspace{-.5em}
    The maximum exists, since $V(M')$ is finite.
  \end{definition}

The interpretation is as follows:
We capture how
 different the distribution in the high-level model is compared to
picking a corresponding low-level intervention ({\color{colB}$f^i$}),
 considering its implementation in $M$ ({\color{colC}$I^M_{f^{-1}_V(\{v_i\})}$}),
 and measuring on the high-level ({\color{colD}$f^m$}).

We prove that
this 
    notion of error is compositional:
  \begin{proposition}[Transformation error is compositional]
    \label{prop:error_composes}
    Let $f: M \to M', g: M' \to M''$ be transformations between models in $\FinMod$.\\
    Then
    \(\error(fg) \leq \error(f) + \error(g).\)
  \end{proposition}
  \begin{proof}
    Let $\{v_{i} = x_{i}\}$ be any intervention in $M''$, and consider this diagram:
    \vspace{-.5em}
    \begin{center}
      \begin{tikzcd}
        {\prod_{i\in I} \mathbf{X}^{M''}_{v_i}} \ar[d] \ar[r] & {\prod_{v \in V(M'')}\mathbf{X}^{M''}_v}
        \ {\footnotesize \rotatebox[origin=c]{15}{\textcolor{gray}{\textit{high}}}}
        \\
        {\prod_{i \in I}\prod_{f(v) = v_i}\mathbf{X}^{M'}_v} \ar[d] \ar[r] & {\prod_{v \in V(M')} \mathbf{X}^{M'}_v}
        \ {\footnotesize \rotatebox[origin=c]{15}{\textcolor{gray}{\textit{mid}}}}
        \ar[u]\\
        {\prod_{i \in I}\prod_{f(v) = v_i}\prod_{g(v') = v}\mathbf{X}^M_{v'}} \ar[r] & \prod_{v \in V(M)} \mathbf{X}^M_v
        \ {\footnotesize\rotatebox[origin=c]{15}{\textcolor{gray}{\textit{low}}}}
        \ar[u] \\[-2.2em]
        \ {\footnotesize\rotatebox[origin=c]{15}{\textcolor{gray}{\textit{intervention}}}}
        &
        \ {\footnotesize\rotatebox[origin=c]{15}{\textcolor{gray}{\textit{distribution}}}}
      \end{tikzcd}
    \end{center}

    By assumption, the failure of the top diagram to commute
    is $\leq \error(f)$ and that of the bottom diagram is
    $\leq \error(g)$,
    so the failure of the composite to commute is $\leq \error(f) + \error(g)$.
    Since this holds for an arbitrary intervention, we have $\error(fg) \leq \error(f)+\error(g)$.
  \end{proof}

\subsection{Notions of abstraction and error}

For the Jensen-Shannon distance
\[d_{\JSD}: \FinStoch(\mathbf{X},\mathbf{Y})^{2} \to [0,\infty]\]
the following are the key properties
for our development of a compositional account of causal model transformations:
\begin{description}
  \item[reflexivity] $d_{\JSD}(f,f) = 0$
  \item[triangle inequality]
  for any kernels $f_{1},f_{2},f_{3}: \mathbf{X} \to \mathbf{Y}$, we have
  $d_{\JSD}(f_{1},f_{3}) \leq d_{\JSD}(f_{1},f_{2}) + d_{\JSD}(f_{2},f_{3})$
  \item[compositionality] for any $g_{1},g_{2} : \mathbf{X} \to \mathbf{Y}, f: \mathbf{Y} \to \mathbf{Z}$
  and $h: \mathbf{W} \to \mathbf{X}$, we have
  $d_{\JSD}(fg_{1},fg_{2}) \leq d_{\JSD}(g_{1},g_{2})$ and $d_{\JSD}(g_{1}h,g_{2}h) \leq d_{\JSD}(g_{1},g_{2})$
\end{description}

It is reasonable that $d_{\JSD}(f,g) = 0 \Rightarrow f=g$,
    which ensures that if there is \emph{no} error,
    the two distributions in question can be considered indistinguishable.
This also rules out pathological distances with $d(f,g) = 0$ for all $f,g$.

It is fruitful to discuss possible relaxations of the underlying distance
    when developing a compositional framework such as ours.
For example, the symmetry $d_{\JSD}(f,g)=d_{\JSD}(g,f)$ is not essential.
The interpretation of any chosen distance $d(f,g)$ in the definition of an error measure 
    is ``How bad is it to predict $f$ when the true distribution is $g$?''.
For this, it may be reasonable for the underlying notion to be asymmetric.
It may also be reasonable to replace the triangle inequality with modified versions like
$d(f,h)^{p} \leq d(f,g)^{p} + d(g,h)^{p}$ for some $p>0$.
For example, one could use the Jensen-Shannon divergence directly,
which satisfies this inequality for $p=1/2$.
The key requirement for compositionality is a useful bound on $d(f,h)$ in terms of $d(f,g)$ and $d(g,h)$.

The compositionality requirement on the distance may not be obvious.
Indeed, we may expect most natural notions of distance to satisfy it.
For a counterexample, consider variables valued in metric spaces and a Wasserstein distance.
Here, if one allows any kernel between metric spaces, compositionality fails:
if $\mathbf{X} = \{1,2\}, d_{\mathbf{X}}(1,2) = 1$, $f_{i}: * \to \mathbf{X}$ is the point-distribution at $i$,
then the sup-Wasserstein distance of $f_{1}$ and $f_{2}$ is $1$.
If we post-compose this with the identity map $i: \mathbf{X} \to \mathbf{Y}$,
where $Y = \{1,2\}, d_{\mathbf{Y}}(1,2) = 2$, then $d_\mathbf{Y}(if_{1},if_{2}) = 2 \centernot\leq 1= d_\mathbf{X}(f_1,f_2)$.
This is commonly avoided by restricting to
\emph{short} maps, that is, to those kernels
$g: \mathbf{X} \to \mathbf{Y}$ such that,
for each $x,x' \in \mathbf{X}$, $d_\mathbf{Y}(g(x),g(x')) \leq d_\mathbf{X}(x,x')$.
It is possible to replace $\JSD$ with another metric.
For example, the metric considered by \citep{beckers2019approximate}
    is induced by a chosen metric on the values of the random variables.
In our terms, this amounts to replacing $d_{\JSD}$
    with the 1-Wasserstein (or Kantorovich) distance,
    and replacing finite sets with a category of metric spaces.
One can show that
    the desired compositionality property holds
    as long as the maps between spaces are required to be short
    (cf.~\cref{prop:short} and see, for example, \citet{paolotobiasmet}).
This shortness requirement is not surprising:
    if the map $f: X \to Y$ maps two points that are indistinguishable
    to two completely different points,
    we cannot expect that this map preserve distances between distributions.
One could develop an analogous version of our theory for metric spaces;
in favour of a clear and concise exposition of our conceptual contribution
we refrain from this development here.

\section{Functors on abstractions}
\label{sec:functors}

Until now, we have considered the problem of comparing two models
    and developed a compositional approach to
    measuring the error incurred by replacing one model by another.
The discussion so far has considered this error \emph{relative to the original model}:
    for a transformation $f: M \to M'$; we 
    compare the prediction made by $M'$
    to the prediction made by $M$ and translated by $f$ into
    a prediction about the variables in $M'$.
This analysis elides the fact that both $M$ and $M'$ are imperfect approximations of reality.
In fact, $M'$ may be closer to reality than $M$ even if $\error(f) > 0$.

In most cases, however, modelling does not start from a ground-truth model,
    which we then seek to approximate. 
Instead, we may have different models for the same real-world system
    that are related to each other by transformations.
The errors of each model may be empirically estimated by performing experiments.
Here we discuss what can be said about the ``true'' error of some $M'$
    in terms of the ``true'' error of another $M$ and $\error(f)$ of $f:M\to M'$.

We make the following informal definition:
An \emph{implemented model} $M$ consists of a finite interventional causal model (also called $M$),
some specified procedure for obtaining a measurement of the variables in the model,
and a way of physically implementing each intervention in the model.

The interpretation of the above definitions is that
    an implementation of a model is an imaginary transformation from a somewhat
    ``idealized model'' representing ``ground truth''.
We can capture how well this implemented model describes reality
by an operational definition of error in terms of a two-player recognition game:
\begin{enumerate}
    \item First, player A chooses some intervention $i$, which is shown to player B.
    \item Then, according to a fair coin flip, player A either 
    physically implements the given intervention and measures the variables in real life, 
    or they sample the variable values according to the 
    interventional distribution
    $P_{M}(- \mid \doop(i))$
    described by the finite interventional causal model $M$.
    \item The outcome of step 2.\ (but not of the coin flip)
    is also revealed to player B.
    Player B must now choose a subjective probability $p \in [0,1]$
    that the measurement was taken in real life.
    \item If the measurement was real, player B scores $1 - \log_{2} p$, else they score $1 - \log_{2} (1-p)$.
\end{enumerate}

The optimal expected score for player B is obtained by choosing for $p$ the conditional probability
of the measurement being taken in real given the variable values; 
the optimal expected score is the mutual information between the coin flip and the values revealed to player B, that is, the Jensen-Shannon distance between real and model distribution~\cite{properscoring}
In principle, this recognition game allows to empirically assess the error of an implemented model
through interventional experimentation.

Let $M,M'$ be two implemented models and define
an \emph{implementation-preserving transformation} as a
model transformation $f: M \to M'$ satisfying both of the following conditions:
\begin{enumerate}
\setlength\itemsep{0em}
  \item For each intervention $i$ in $M'$,
  its implementation has the same effect as choosing an intervention in $M$ according to the distribution $f^{i}(i)$
  and implementing that.
  \item Measuring a variable $x$ in $M'$ is the same as measuring $f_{V}^{-1}(\{x\})$ in $M$ and then applying $f^{m}$.
\end{enumerate}
Here, the terms ``has the same effect'' and ``the same'' are to be informally understood.
We assume, for example,
    that the experimenter can
    sample a random intervention from $f^{i}(i)$
    in a way that does not interfere with the experiment itself.

Based on this, we might expect the following to hold:
\begin{itemize}
\item Given an implemented model $M$ and a transformation $f: M \to M'$ to some model $M'$,
  there is a unique implementation of $M'$ such that $f$ is implementation-preserving.
\item Given an implementation-preserving transformation $f: M \to M'$,
  the error of $M'$ is at most $\error(f)$ greater than the error of $M$.
\end{itemize}

Category theorists may recognise this as a \emph{functor}:
The assignment to each abstract model $M$ of its set of implementations and their corresponding error
    is a functor from $\FinMod$ to a category of error spaces (cf.~\cref{sec:enriched}).

The ``ground truth'' functor above is inaccessible to mathematical description,
but we can consider variants of this idea.
For example, if $M$ is a model, the assignment
\[M' \mapsto \big(\{f: M \to M'\}, f \mapsto \error(f) + \epsilon\big)\]
is a functor.
This corresponds to a situation where we know that reality is described by the model $M$
    with error at most $\epsilon$,
    while we are without access to ground truth itself. 
Nonetheless,
    we can evaluate models by only comparing them to $M$;
    we have to add an extra $\epsilon$ of error
    to account for the fact that $M$ itself is an imperfect approximation of reality.

The operational meaning is: If a transformation $M \to M'$ has error $\epsilon$,
    then the expected score of player B in the recognition game for $M'$ can be bounded in terms of the expected score of player B in the recognition game for $M$ and $\epsilon$,
    namely $\sqrt{e(M')} \leq \sqrt{e(M)} +\epsilon$.
This way, the error of a transformed model in describing reality
can be bounded without additional interventional real-world experiments.

The above analysis hinges on the compositionality of abstraction error,
    such that the error $\error(f)$ provides a bound
    on how much the error of the implemented model $M'$ is increased compared to $M$.
This underlines the fruitfulness of compositionality as a desideratum
    and of the categorical abstract viewpoint:
We can reasonably talk about bounding the error of a causal model relative
    to ground truth as long as we have a reference model $M$
    for which we evaluated the error
    by the intervention procedure outlined above.

\section{Discussion}\label{sec:discussion}

We provided a categorical perspective on causal model transformations.
Our approach is based on 
    a category of finite interventional causal models
    and satisfies an important desideratum:
    compositionality of model transformations and the associated approximation errors.
While we consider a category of causal models and transformations between them,
    existing work
    on the application of category theory to the domain of causal modelling
    has studied one causal model via categorical tools.
\citet{fong2013causal}, for example, develops the theory of directed acyclic graph models
    using syntactical categories.
This is further developed by \citet{Jacobs_2019},
    who demonstrate how to
    carry out causal inference
    by string diagram manipulations.
\citet{fritz2019synthetic} lays out the foundations for
    an ambitious programme of developing probability theory in the language of categories.
Working in this framework, \cite{patterson2020}
    develops an analogy between statistics and universal algebra,
    where a statistical model becomes a model of a theory, in the sense of logic. 
This separation of theory and model is akin to 
    the approach presented by \citet{bongers2018theoretical}:
    detaching the structural equations (the \emph{theory})
    from the random variables that simultaneously (almost surely)
    solve those equations (the \emph{model}),
    they provide a measure theoretic treatment of
    cyclic models.

Conceptualising a framework for causal model transformations
    can be motivated from different vantage points.
First, it helps characterise when observable variables may be
    ill-suited for a causal description
    by viewing the observables as a transformation
    of underlying causal entities with high error.
For example, in the analysis of electroencephalographic data we may wish to recover
    signals that correspond to cortical activity
    instead of reasoning about interventions on mixed electrode signals
    \citep{weichwald2016merlin}.
Second, if observables are ill-suited for a causal description,
    we may wish to find a transformation that yields variables
    amenable to a causal description.
\citet{chalupka2015visual}, for example,
    present how to learn the macroscopic visual cause
    of some behaviour from observed pixel values.
Third, we may be interested in abstracting or aggregating 
    information 
    to obtain a macro-level description of a system
    that is pragmatically more useful as it
    represents the information necessary for a certain task
    more clearly than a complex fine-grained model~\citep{hoel2013quantifying,hoel2017when,kinney2020causal,weichwald2019pragmatism}.
Last,
    approaches to infer causality between latent causal variables
    based on observed variables or time-subsampled observations
    may be embedded within a framework of causal model transformation
    where transformations encode which
    variables or time-points are unobserved~\citep{hyttinen2016causal,silva2006learning}.

    Our category theoretic framework of causal model transformations
    is instructive to
    clarify the assumptions and arguments required to proof its compositionality:
    We require the distance between kernels
    to be compatible with composition
    of kernels, that is, beyond the triangle inequality
    we require analogs of \cref{prop:short,lemma:jsdcomposes}.
    This condition is natural from a category-theoretical perspective.
    The formal tools of category theory
    enable diagrammatic reasoning
    and
    a simple proof that the resulting framework of causal model transformations and their abstraction errors
    is compositional.

\begin{acknowledgements}
    We thank the anonymous reviewers
    for their constructive comments
    that helped improve the interdisciplinary presentation.
    SW was supported by the Carlsberg Foundation.
\end{acknowledgements}

\clearpage
\bibliography{references}

\bibliographystyle{plainnat}

\doparttoc
\faketableofcontents

\onecolumn
\maketitle
\appendix

\addcontentsline{toc}{section}{Appendix}
\renewcommand \thepart{}
\renewcommand \partname{}
\part{Appendix}
\parttoc

\clearpage
\section{KL-divergence is arbitrarily far from satisfying the triangle inequality}
\label{appendix_kl_counterexample}

We recall the definition of Kullback-Leibler divergence.
\begin{definition}[Kullback-Leibler divergence]
For discrete probability distributions $q,p$ on the same probability space $\mathbf{X}$,
the KL-divergence from $q$ to $p$ is defined as
  $D(p|q) = \sum_{x\in\mathbf{X}} p(x)\ln\left(\frac{p(x)}{q(x)}\right)$.
\end{definition}

\begin{proposition}[KL-divergence fails the triangle inequality]
  For any $\epsilon > 0$, there exist probability measures $p,q,r$ on $\mathbb{N}$ so that
  \[ D(q|p) < \epsilon \text{ and } D(r|q) < \epsilon \text{ while } D(r|p) = \infty.\]
\end{proposition}

\begin{proof}

Let $p(x) = 1/2^{x}$, $q(x) = Z/x^{3}$, and $r(x) = Z'/x^{2}$,
where $Z,Z'$ are normalization constants so that
\[\sum_{x=1}^{\infty}q(x) = \sum_{x=1}^{\infty}r(x) = 1.\]

Observe that
\begin{align*}
D(q|p) &= \sum_{x\in\mathbb{N}} q(x)\ln\left(\frac{q(x)}{p(x)}\right)
    = \sum_{x\in\mathbb{N}} \ln\left(\frac{2^{x}Z}{x^{3}}\right)\frac{Z}{x^{3}} \\
\text{ and }\quad
D(r|q) &= \sum_{x\in\mathbb{N}} r(x)\ln\left(\frac{r(x)}{q(x)}\right)
    = \sum_{x\in\mathbb{N}} \ln\left(\frac{x^{3}Z'}{x^{2}Z}\right)\frac{Z'}{x^{2}},
\end{align*}
both series of which are convergent.
Thence, $D(q|p)$ and $D(r|q)$ are both bounded (the specific bound is not important here).

On the other hand,
\[D(r|p) = \sum_{x\in\mathbb{N}}r(x)\ln\left(\frac{r(x)}{p(x)}\right)
    = \sum_{x\in\mathbb{N}} \ln\left(\frac{2^{x}Z'}{x^{2}}\right)\frac{Z'}{x^{2}},\]
which is \emph{divergent}.
Thence, $D(r|p) = \infty$.

Now for each $\alpha \in [0,1]$,
let $p_{\alpha}(x) = Y_{\alpha}\left({\frac{\alpha}{p(x)} + \frac{1-\alpha}{q(x)}}\right)^{-1}$,
and $r_{\alpha}(x) = \alpha r(x) + (1-\alpha) q(x)$,
where $Y_{\alpha}$ is a suitable normalization constant so that these are probability distributions.
The series involved in computing the normalisation constant for $p_{\alpha}$
is convergent by the harmonic-arithmetic inequality, which also implies that $Y_{\alpha} \geq 1$.
$Y_{\alpha}$ is a continuous function of $\alpha$.

Now the claim is that $D(q| p_{\alpha}), D(r_{\alpha}|q) \to 0$ as $\alpha \to 0$,
while $D(r_{\alpha}|p_{\alpha}) =\infty$ for all $\alpha > 0$.

First consider
\[D(q|p_{\alpha}) = \sum_{x\in\mathbb{N}} \frac{Z}{x^{3}}
  \ln\left(Z\frac{ \alpha 2^{x}+(1-\alpha)x^{3}/Z}{Y_{\alpha}x^{3}}\right)\]

Since $\ln$ is increasing, we can bound this sum by
\[\sum_{x\in\mathbb{N}} \frac{Z}{x^{3}}\ln\left(Z\frac{\max(2^{x},x^{3}/Z)}{Y_{\alpha}x^{3}}\right)\]
Since $2^{x}Z$ dominates for $x$ large enough, this sum converges.
Let $\delta > 0$ be given.
Choose $M$ large enough that $\sum_{x=M}^{\infty}\frac{Z}{x^{3}}\ln\left(Z\frac{\max(2^{x},x^{3}/Z)}{Y_{\alpha}x^{3}}\right) < \delta/2$.
Note that $M$ is independent of $\alpha$ – this is possible because $Y_{\alpha} \geq 1$. Then we also have
\[ \sum_{x=M}^{\infty} \frac{Z}{x^{3}}
  \ln\left(Z\frac{\alpha 2^{x}+(1-\alpha)x^{3}/Z}{Y_{\alpha}x^{3}}\right) < \delta/2\]
Now each term of the sum goes to zero as $\alpha \to 0$, since $Y_{0} = 1$ and $\alpha \mapsto Y_{\alpha}$ is continuous.
Hence we can choose $\alpha$ small enough that the sum of the first $M-1$ terms is $< \delta/2$.
This proves that the whole sum goes to $0$ as $\alpha \to 0$.

For the case of $D(r_{\alpha}|q)$, we have
\[
D(r_{\alpha}|q) =
\sum_{x\in\mathbb{N}} \big(\alpha r(x) + (1-\alpha) q(x)\big)
\ln\left(\frac{\alpha r(x) + (1-\alpha)q(x)}{q(x)}\right)\]
We can use an analogous argument.
We can use monotonicity and the fact that $\frac{Z'}{x^{2}}$ is eventually bigger than $\frac{Z}{x^{3}}$ 
to prove convergence, with convergence speed independent of $\alpha$, 
so that the tail is eventually $<\delta/2$ independently of $\alpha$, then choose $\alpha$ to bound the head.

Now consider $D(r_\alpha|p_\alpha)$, which we can rewrite as
\[
D(r_{\alpha}|p_{\alpha}) =
\sum_{x\in\mathbb{N}} \left(\frac{\alpha Z'}{x^{2}} + \frac{(1-\alpha)Z}{x^{3}}\right)
\ln\left(\frac{\frac{\alpha Z'}{x^{2}}+\frac{(1-\alpha)Z}{x^{3}}}{\frac{Y_{\alpha}}{\alpha 2^{x} + (1-\alpha)x^{3}/Z}}\right).\]

First, we see that this is larger than $\sum_{x} \frac{\alpha Z'}{x^{2}}\ln\left(\frac{\frac{\alpha Z'}{x^{2}}}{\frac{Y_{\alpha}}{\alpha 2^{x} + (1-\alpha)x^{3}}}\right).$

We can rewrite this as

\[\sum_{x\in\mathbb{N}} \frac{\alpha Z'}{x^{2}} \ln(aZ'\frac{\alpha 2^{x} + (1-\alpha)x^{3}}{Y x^{2}})\]
\[\geq \sum_{x\mathbb{N}} \frac{\alpha Z'}{x^{2}} \ln\left(\alpha Z'\frac{\alpha 2^{x}}{Y_{\alpha} x^{2}}\right)\]

As above, this diverges independently of the value of $\alpha$ (as long as $\alpha > 0$).
Hence $D(r_{\alpha}|p_{\alpha}) = \infty$ for all $\alpha > 0$.

Hence for sufficiently small but positive $\alpha$, 
the distributions $p_{\alpha},q,r_{\alpha}$ satisfy the desired properties.

\end{proof}

One might hope that this rather stark result is just a quirk of the infinities involved,
disappearing when we restrict to finite probability measures.
Since KL-divergences on finite sets are always \emph{finite} (as long as the measures involved have the same support),
we cannot reproduce the infinity in the above result on finite sets.
However, we can come arbitrarily close, in the following sense:

\begin{proposition}[Instance of KL-divergence failing]
  For any $\epsilon, K > 0$,
  there exist a finite set $\{1,\dots N\}$ and probability measures $p,q,r$
  on it so that $D(q|p),D(r|q) < \epsilon, D(r|p) > K$.
\end{proposition}

\begin{proof}
Let $p_{\alpha},q,r_{\alpha}$ be as before. 
Let $p_{\alpha}^{N},q^{N},r_{\alpha}^{N}$ be the distributions on $\{1,\dots N\}$ 
so that $p_{\alpha}^{N}(x) = p_{\alpha}(x)$ for $x \in \{1,\dots N-1\}$, etc.
Then $D(p_{\alpha}^{N}|r_{\alpha}^{N}) \to \infty$ as $N \to \infty$ – the sum computing this divergence is $N-1$ terms of the sum computing the overall (infinite) divergence, and a remainder term which is certainly positive for $N$ large enough.

On the other hand, we now show that $D(p_{\alpha}^{N}|q^{N}) \to D(p_{\alpha}|q)$ as $N \to \infty$.
To see this, consider the sum in question:
\[D(p_{\alpha}^{N}|q^{N}) = \sum_{x=1}^{N-1}p_{\alpha}(x)\ln(p_{\alpha}(x)/q(x)) +
\left(\sum_{x=N}^{\infty}p_{\alpha}(x)\right)\ln\left(\frac{\sum_{x=N}^{\infty}p_{\alpha}(x)}{\sum_{x=N}^{\infty}q(x)}\right)\]
The first term here converges to $D(p_{\alpha}|q)$, so it suffices to show that the remainder converges to zero.

We can write this remainder as $(\sum_{x=N}^{\infty}p_{\alpha}(x))(\ln(\sum_{x=N}^{\infty}p_{\alpha(x)}) - \ln(\sum_{x=N}^{\infty}q(x)))$

Observe that we can write
\[D(p_{\alpha}|q) = \sum_{x} p_{\alpha}(x)\ln(p_{\alpha}(x)) - \sum_{x} p_{\alpha}(x)\ln(q(x))\]
Since the left-hand side is finite, and the first sum is convergent, so is the second sum.

Hence we can use convexity of $\ln$ to obtain the following inequality:
\[\leq \left(\sum_{x=N}^{\infty}p_{\alpha}(x)\right)\ln\left(\sum_{x=N}^{\infty}p_{\alpha}(x)\right) -\sum_{x=N}^{\infty}p_{\alpha}(x)\ln(q(x))\]

Now convergence means that as $N \to \infty$ the second term goes to zero.
And since $\sum_{x=N}^{\infty}p_{\alpha}(x) \to 0$, and $x \ln(x) \to 0$ when $x \to 0$, so does the first term.

The analogous argument verifies the same statement for $D(q^{N}|r^{N}_{\alpha})$.

Thus, by choosing $\alpha$ small enough and $N$ big enough, we obtain the desired measures.
\end{proof}

\clearpage
\section{String-diagrammatic construction of interventional distributions}
\label{sec:diagrammatic_proofs}

In \cref{def:interventionaldists},
we claim that an interventional model has a well-defined interventional distribution,
a kernel $I_{S}: \prod_{v \in S} \mathbf{X}_{v}^{M} \to \prod_{v \in V({M})} \mathbf{X}^{M}_{v}$ for any subset $S \subseteq V({M})$.

The intuition is as follows.
To sample $(y_{v})_{v \in V({M})}$ according to $I_{S}(x_{v})_{v \in S}$,
we do the following:
\begin{enumerate}
  \item If $v \in S$, $y_{v} = x_{v}$ with probability $1$.
  \item Identify a variable $v$ so that $y_{v}$ has not been determined yet,
  but all its parents have been determined,
  and sample $y_{v}$ according to $\phi^{M}_{v}((y_{v'})_{v' \to v})$.
\end{enumerate}

The question is whether this gives a well-defined distribution.
We here prove that it does.

For convenience,
we write $\mathbf{X}^{M}_{X} = \prod_{v \in X}\mathbf{X}^{M}_{v}$
when $X \subset V({M})$ is a subset of variables of $M$
and
$\pa(v) \subseteq V({M})$ for the set of parents of $v \in V({M})$.

We use the graphical notation known as \emph{string diagrams}.
These are widely used in category theory to depict constructions in \emph{monoidal categories}.
A full discussion of the technical details behind their use is beyond the scope of this appendix
(see, for example, \citet{selinger}).
Their meaning in the special case under consideration, kernels between finite sets,
can be intuitively understood.
For example, the following diagram (read bottom-to-top)\\[.5em]
\centerline{\scalebox{0.7}
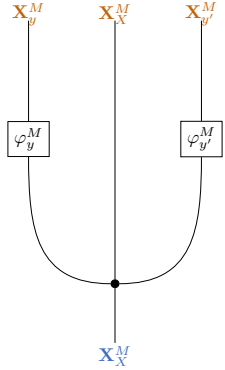
{\begin{tikzpicture}
	\begin{pgfonlayer}{nodelayer}
		\node [style=none] (48) at (12.5, -10.5) {{\color{colA}$\mathbf{X}^M_X$}};
		\node [style=none] (50) at (15.75, 2.5) {{\color{colB}$\mathbf{X}^M_{y'}$}};
		\node [style=bn] (51) at (12.5, -7.75) {};
		\node [style=none] (54) at (9.25, 2.5) {{\color{colB}$\mathbf{X}^M_y$}};
		\node [style=none] (55) at (12.5, 2.5) {{\color{colB}$\mathbf{X}^M_X$}};
		\node [style=morphism] (56) at (9.25, -2.25) {{$\phi_y^M$}};
		\node [style=none] (58) at (12.5, -10) {};
		\node [style=none] (59) at (12.5, 2.25) {};
		\node [style=none] (60) at (9.25, 2.25) {};
		\node [style=none] (61) at (15.75, 2.25) {};
		\node [style=morphism] (62) at (15.75, -2.25) {{$\phi_{y'}^M$}};
	\end{pgfonlayer}
	\begin{pgfonlayer}{edgelayer}
		\draw (58.center) to (51);
		\draw [in=-90, out=180, looseness=1.25] (51) to (56);
		\draw (56) to (60.center);
		\draw (51) to (59.center);
		\draw (62) to (61.center);
		\draw [in=-90, out=0, looseness=1.25] (51) to (62);
	\end{pgfonlayer}
\end{tikzpicture}}}\\[.5em]
depicts a kernel
\({\color{colA}\mathbf{X}_X^{M}} \to  {\color{colB}\mathbf{X}_y^{M}} \times {\color{colB}\mathbf{X}_X^{M}} \times
{\color{colB}\mathbf{X}_{y'}^{M}}\),
informally described as ``given \(x \in \mathbf{X}_X^M\),
sample \(y \in \mathbf{X}_y\) from the distribution
\(\phi^M_y(x)\)
and independently sample \(y'\) from the distribution
\(\phi^M_{y'}(x)\), then return the tuple \((y,x,y')\)''.

\begin{proposition}\label{interv_welldef}
  The interventional distribution in \cref{def:interventionaldists} is well-defined.
\end{proposition}

For this proposition, we need the following lemma.

\begin{lemma}\label{lem_permut_order}
  Let $S$ be a finite partially ordered set.
  Let $A: s_{1} \dots s_{n}$ and $B: s'_{1} \dots s'_{n}$ be
  two totalizations of the ordering on $S$ --- in other words,
  two ways of arranging the elements of $S$ in a non-decreasing sequence.
  Then one can turn $A$ into $B$ by a finite sequence of transpositions, where each transposition exchanges two adjacent, incomparable elements.
\end{lemma}
\begin{proof}
  Let's show that any nondecreasing sequence can be turned into $B$ by such a sequence of transpositions --- this is really the content of the lemma.
  Define the error of a sequence $s_{1} \dots s_{n}$ as the total number of pairs $i,j$ so that $s_{i}$ and $s_{j}$ are not in the same order as in $B$.
  If the error is zero, we must already be in sequence $B$.
  Suppose the error is greater than zero.
  Then there must be a pair of consecutive elements, $s_{i},s_{i+1}$, that are in the wrong order compared to the ordering $B$.
  The elements must also be incomparable:\\
  \hspace*{1em}1) we cannot have $s_{i+1} \leq s_{i}$, since it is a non-decreasing sequence, \\
  \hspace*{1em}2) we cannot have $s_{i} \leq s_{i+1}$ —
  if this was true, the pair would already be in the same order as in $B$.\\
  Hence we can swap $s_{i}$ and $s_{i+1}$, which decreases the error by $1$.
  After a finite number of steps the error must be zero and we have obtained $B$.
\end{proof}

\begin{proof}[Proof of \cref{interv_welldef}]
  By applying \cref{lem_permut_order} to the vertices of the DAG $G^M$,
  partially ordered by causal dependence,
  we see that we can move between any two constructions of
  the interventional distribution by swapping two consecutive variables at a time.
  Hence it suffices to show that we may swap the order of two consecutive $y_{i}$,
  neither dependent on the other,
  without changing the final distribution.
  Consider the following diagram manipulation:
\begin{equation}\label{eq:varswap}
\scalebox{0.7}{\begin{tikzpicture}
	\begin{pgfonlayer}{nodelayer}
		\node [style=none] (0) at (-12, -10.25) {{$\mathbf{X}^M_X$}};
		\node [style=none] (1) at (-10.75, -4) {{$\mathbf{X}^M_y$}};
		\node [style=none] (2) at (-9, 2.5) {{$\mathbf{X}^M_{y'}$}};
		\node [style=bn] (3) at (-12, -7.5) {};
		\node [style=bn] (4) at (-12, -2.75) {};
		\node [style=bn] (5) at (-15, -2.75) {};
		\node [style=none] (6) at (-12, 2.5) {{$\mathbf{X}^M_y$}};
		\node [style=none] (7) at (-15, 2.5) {{$\mathbf{X}^M_X$}};
		\node [style=morphism] (8) at (-12, -5) {{$\phi_y^M$}};
		\node [style=medium box] (9) at (-9, 0.25) {{$\phi_{y'}^M$}};
		\node [style=none] (10) at (-12, -9.75) {};
		\node [style=none] (11) at (-15, 2.25) {};
		\node [style=none] (12) at (-12, 2.25) {};
		\node [style=none] (13) at (-9, 2.25) {};
		\node [style=none] (14) at (-8.75, 0.25) {};
		\node [style=none] (15) at (-9.25, 0.25) {};
		\node [style=none] (16) at (0, -10.25) {{$\mathbf{X}^M_X$}};
		\node [style=none] (17) at (1.25, -4) {{$\mathbf{X}^M_y$}};
		\node [style=none] (18) at (3, 2.5) {{$\mathbf{X}^M_{y'}$}};
		\node [style=bn] (19) at (0, -7.5) {};
		\node [style=bn] (21) at (-3, -2.75) {};
		\node [style=none] (22) at (0, 2.5) {{$\mathbf{X}^M_y$}};
		\node [style=none] (23) at (-3, 2.5) {{$\mathbf{X}^M_X$}};
		\node [style=morphism] (24) at (0, -5) {{$\phi_y^M$}};
		\node [style=medium box] (25) at (3, 0.25) {{$\phi_{y'}^M$}};
		\node [style=none] (26) at (0, -9.75) {};
		\node [style=none] (27) at (-3, 2.25) {};
		\node [style=none] (28) at (0, 2.25) {};
		\node [style=none] (29) at (3, 2.25) {};
		\node [style=none] (46) at (-6, -3) {{$=$}};
		\node [style=none] (47) at (6, -3) {{$=$}};
		\node [style=none] (48) at (12.5, -10.5) {{$\mathbf{X}^M_X$}};
		\node [style=none] (50) at (15.75, 2.5) {{$\mathbf{X}^M_{y'}$}};
		\node [style=bn] (51) at (12.5, -7.75) {};
		\node [style=none] (54) at (9.25, 2.5) {{$\mathbf{X}^M_y$}};
		\node [style=none] (55) at (12.5, 2.5) {{$\mathbf{X}^M_X$}};
		\node [style=morphism] (56) at (9.25, -2.25) {{$\phi_y^M$}};
		\node [style=none] (58) at (12.5, -10) {};
		\node [style=none] (59) at (12.5, 2.25) {};
		\node [style=none] (60) at (9.25, 2.25) {};
		\node [style=none] (61) at (15.75, 2.25) {};
		\node [style=morphism] (62) at (15.75, -2.25) {{$\phi_{y'}^M$}};
	\end{pgfonlayer}
	\begin{pgfonlayer}{edgelayer}
		\draw (10.center) to (3);
		\draw [in=-90, out=-180, looseness=1.25] (3) to (5);
		\draw (3) to (8);
		\draw (8) to (4);
		\draw [in=-90, out=0, looseness=1.25] (4) to (14.center);
		\draw [in=-90, out=0, looseness=0.75] (5) to (15.center);
		\draw (9) to (13.center);
		\draw (4) to (12.center);
		\draw (5) to (11.center);
		\draw (26.center) to (19);
		\draw [in=-90, out=-180, looseness=1.25] (19) to (21);
		\draw (19) to (24);
		\draw (25) to (29.center);
		\draw (21) to (27.center);
		\draw (24) to (28.center);
		\draw [in=-90, out=0] (21) to (25);
		\draw (58.center) to (51);
		\draw [in=-90, out=180, looseness=1.25] (51) to (56);
		\draw (56) to (60.center);
		\draw (51) to (59.center);
		\draw (62) to (61.center);
		\draw [in=-90, out=0, looseness=1.25] (51) to (62);
	\end{pgfonlayer}
\end{tikzpicture}}
\end{equation}
    In the first step we use the fact that $y'$ does not depend on $y$, so we may delete the $\mathbf{X}^{M}_{y}$ input to $\phi_{y'}^{M}$.
  Then rearrange the wires.

  This shows that the composition $\mathbf{X}^{M}_{X} \to \mathbf{X}^{M}_{X \cup \{y\}} \to \mathbf{X}_{X \cup \{y,y'\}}^{M}$ is equal to another map $\mathbf{X}_{X} \to \mathbf{X}_{X \cup \{y,y'\}}$.
  A similar argument shows that the composition $\mathbf{X}_{X} \to \mathbf{X}_{X \cup \{y'\}} \to \mathbf{X}_{X \cup \{y,y'\}}^{M}$ is equal to the same map.
  This concludes our proof.
\end{proof}

We also prove our claim that this distribution is ``the right one'',
in the sense that the mechanisms are the conditional distributions.
We introduce the following diagram-theoretic definition of conditionals:
\begin{definition}\label{conditional_def}
  Let $\psi: \spc{A} \to \spc{X} \times \spc{Y}$ be a Markov kernel.
  We say that a kernel $p : \spc{A} \times \spc{X} \to \spc{Y}$
  is a \emph{conditional distribution of $Y \in \spc{Y}$ given $A \in \spc{A}$ and $X \in \spc{X}$},
  if there exists $\phi: \spc{A} \to \spc{X}$ so that we have the following identity:
  \begin{equation}\label{eq_condeq}
    \begin{tikzpicture}
	\begin{pgfonlayer}{nodelayer}
		\node [style=none] (0) at (-4, -2.75) {{$\spc{A}$}};
		\node [style=none] (1) at (-4.75, 2.25) {{$\spc{X}$}};
		\node [style=none] (2) at (-3.25, 2.25) {};
		\node [style=none] (3) at (-3.25, 2.25) {{$\spc{Y}$}};
		\node [style=none] (4) at (-4.75, 1.75) {};
		\node [style=none] (5) at (-3.25, 1.75) {};
		\node [style=none] (6) at (-4, -2.25) {};
		\node [style=medium box] (8) at (-4, -0.75) {{$\psi$}};
		\node [style=none] (9) at (-3.5, -0.75) {};
		\node [style=none] (10) at (-4.5, -0.75) {};
		\node [style=none] (11) at (2.75, -2.75) {{$\spc{A}$}};
		\node [style=morphism] (12) at (2.75, -1) {{$\phi$}};
		\node [style=bn] (13) at (2.75, -0.25) {};
		\node [style=morphism] (14) at (3.75, 0.75) {{$p$}};
		\node [style=none] (15) at (3.75, 1.75) {};
		\node [style=none] (16) at (1.75, 1.75) {};
		\node [style=none] (17) at (2.75, -2.25) {};
		\node [style=none] (18) at (1.75, 2.25) {{$\spc{X}$}};
		\node [style=none] (19) at (3.75, 2.25) {{$\spc{Y}$}};
		\node [style=none] (20) at (-0.5, -0.5) {{$=$}};
		\node [style=bn] (22) at (2.75, -1.75) {};
		\node [style=none] (23) at (4, 0.75) {};
	\end{pgfonlayer}
	\begin{pgfonlayer}{edgelayer}
		\draw (6.center) to (8);
		\draw [in=-90, out=90, looseness=1.25] (10.center) to (4.center);
		\draw [in=-90, out=90, looseness=1.25] (9.center) to (5.center);
		\draw (12) to (13);
		\draw [in=-90, out=180, looseness=1.25] (13) to (16.center);
		\draw [in=-90, out=0, looseness=1.25] (13) to (14);
		\draw (14) to (15.center);
		\draw (17.center) to (22);
		\draw (22) to (12);
		\draw [in=0, out=-90] (23.center) to (22);
	\end{pgfonlayer}
\end{tikzpicture}
  \end{equation}
\end{definition}

\begin{remark}
  \Cref{conditional_def} is a definition of conditional distributions suitable for parameterized joint distributions.
  Dealing with such distributions is necessary if we want to combine conditional and interventional distributions.
  In the case $\spc{A} = \{*\}$, we recover the usual situation of a joint distribution on a product set.

  Let us spell out the connection with the normal definition of conditional distribution:
  a map $p: \spc{A} \times \spc{X} \to \spc{Y}$ is a conditional distribution for $\psi: \spc{A} \to \spc{X} \times \spc{Y}$ if and only if, for all $a \in \spc{A}$,
  and for all $x \in \spc{X}$ with nonzero probability given $a$, the distribution $p(a,x)$ is the conditional distribution of $Y$ given $X=x$ and $(X,Y) \sim \psi(a)$.
  This is also the reason we say \emph{a} conditional distribution and not \emph{the} conditional distribution.

  For a more thorough discussion of this point from a categorical point of view,
  see, for example,~\citet[Section 11
  (in particular Definition 11.5 and Remark 11.6)]{fritz2019synthetic}.
\end{remark}

\begin{proposition}\label{mech_is_cond}
  Each mechanism $\phi_{X}^{M} : \mathbf{X}^{M}_{\pa(v)} \to \mathbf{X}^{M}_{v}$ is a conditional distribution for the observational distribution
  $* \to \mathbf{X}^{M}_{\pa(V)} \times \mathbf{X}^{M}_{v}$.
\end{proposition}
\begin{proof}
  Recall that given a distribution $* \to \spc{X} \tensor \spc{Y}$, a kernel $\spc{X} \to \spc{Y}$ is a conditional distribution if and only
  if we have the identity in \cref{eq_condeq}.
  After marginalizing out the other variables, the observational distribution on $ \mathbf{X}^{M}_{\pa(v)} \times \mathbf{X}_{v}^{M}$ factors as
  \[* \to\mathbf{X}^{M}_{\pa(v)} \overset{(1,\phi^{M}_{v})}{\longrightarrow} \mathbf{X}^{M}_{\pa(v)} \times \mathbf{X}_{v}^{M}.\]
  Diagrammatically, this looks like

  
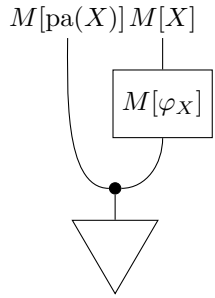
\begin{center}
     \begin{tikzpicture}
	\begin{pgfonlayer}{nodelayer}
		\node [style=medium state] (0) at (-4.5, 0) {};
		\node [style=bn] (1) at (-4.5, 1.5) {};
		\node [style=none] (2) at (-5.75, 5.5) {};
		\node [style=medium box] (3) at (-3.25, 3.75) {{$M[\phi_X]$}};
		\node [style=none] (4) at (-3.25, 5.5) {};
		\node [style=none] (5) at (-5.75, 6) {{$M[\pa(X)]$}};
		\node [style=none] (6) at (-3.25, 6) {{$M[X]$}};
	\end{pgfonlayer}
	\begin{pgfonlayer}{edgelayer}
		\draw (0) to (1);
		\draw [in=-90, out=-180, looseness=0.75] (1) to (2.center);
		\draw [in=-90, out=0] (1) to (3);
		\draw (3) to (4.center);
	\end{pgfonlayer}
\end{tikzpicture}
  \end{center}

  Which is the statement we wanted.
\end{proof}
Here the triangle denotes the observational distribution on $\mathbf{X}^{M}_{\pa(v)}$.

There is nothing surprising about this proposition:
just as in the classical theory of graphical models, it holds by construction.
In classical treatments,
this is usually not made into a theorem,
although it is implicit in most treatments of the Markov property for structure causal models,
see, for example,~\citet[Proposition 6.31]{peters2017elements}.

\begin{proposition}\label{markov_prop}
  The observational distribution satisfies the directed Markov property with respect to the DAG $G$.
\end{proposition}
\begin{proof}
  We must prove that any variable $v$ is independent of its non-descendants given its parents.
  Let us introduce the notation $\nd(v)$ for the non-descendants of $v$ excluding its parents.
  Then we are trying to show $\nd(v) \bot v \mid \pa(v)$.

  Observe this diagram manipulation:
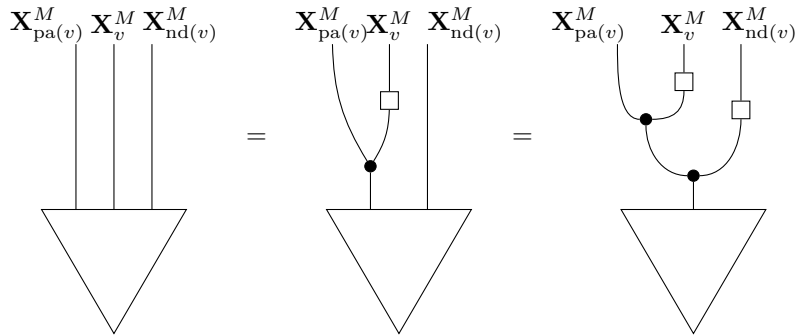
\vspace{0.5cm}
  \begin{center}
   \begin{tikzpicture}
	\begin{pgfonlayer}{nodelayer}
		\node [style=large state] (0) at (-6.75, 0) {};
		\node [style=none] (1) at (-7.75, 0) {};
		\node [style=none] (2) at (-6.75, 0) {};
		\node [style=none] (3) at (-5.75, 0) {};
		\node [style=none] (4) at (-7.75, 5.5) {};
		\node [style=none] (5) at (-6.75, 5.5) {};
		\node [style=none] (6) at (-5.75, 5.5) {};
		\node [style=none] (7) at (-8.5, 6) {{$\mathbf{X}^M_{\pa(v)}$}};
		\node [style=none] (8) at (-6.75, 6) {{$\mathbf{X}^M_{v}$}};
		\node [style=none] (9) at (-5, 6) {{$\mathbf{X}^M_{\nd(v)}$}};
		\node [style=large state] (10) at (0.75, 0) {};
		\node [style=none] (11) at (0, 0) {};
		\node [style=none] (12) at (0.75, 0) {};
		\node [style=none] (13) at (1.5, 0) {};
		\node [style=none] (16) at (1.5, 5.5) {};
		\node [style=none] (17) at (-1, 6) {{$\mathbf{X}^M_{\pa(v)}$}};
		\node [style=none] (18) at (0.5, 6) {{$\mathbf{X}^M_v$}};
		\node [style=none] (19) at (2.5, 6) {{$\mathbf{X}^M_{\nd(v)}$}};
		\node [style=bn] (20) at (0, 2.25) {};
		\node [style=none] (21) at (-1, 5.5) {};
		\node [style=morphism] (22) at (0.5, 4) {};
		\node [style=none] (23) at (0.5, 5.5) {};
		\node [style=none] (24) at (-3, 3) {};
		\node [style=none] (25) at (-3, 3) {{$=$}};
		\node [style=large state] (26) at (8.5, 0) {};
		\node [style=none] (30) at (9.75, 5.5) {};
		\node [style=none] (31) at (5.75, 6) {{$\mathbf{X}^M_{\pa(v)}$}};
		\node [style=none] (32) at (8.25, 6) {{$\mathbf{X}^M_v$}};
		\node [style=none] (33) at (10.25, 6) {{$\mathbf{X}^M_{\nd(v)}$}};
		\node [style=bn] (34) at (7.25, 3.5) {};
		\node [style=none] (35) at (6.5, 5.5) {};
		\node [style=morphism] (36) at (8.25, 4.5) {};
		\node [style=none] (37) at (8.25, 5.5) {};
		\node [style=bn] (38) at (8.5, 2) {};
		\node [style=morphism] (39) at (9.75, 3.75) {};
		\node [style=none] (40) at (4, 3) {{$=$}};
	\end{pgfonlayer}
	\begin{pgfonlayer}{edgelayer}
		\draw (4.center) to (1.center);
		\draw (2.center) to (5.center);
		\draw (3.center) to (6.center);
		\draw (13.center) to (16.center);
		\draw (11.center) to (20);
		\draw [bend left=15] (20) to (21.center);
		\draw [bend right=15] (20) to (22);
		\draw (22) to (23.center);
		\draw [in=-90, out=-180, looseness=0.75] (34) to (35.center);
		\draw [in=-90, out=0, looseness=1.25] (34) to (36);
		\draw (36) to (37.center);
		\draw (26) to (38);
		\draw [in=-90, out=-180] (38) to (34);
		\draw [in=-90, out=0] (38) to (39);
		\draw (39) to (30.center);
	\end{pgfonlayer}
\end{tikzpicture}
  \end{center}
\vspace{0.5cm}
  In the first step, we are factoring the observational distribution on $v\cup\pa(v)\cup\nd(v)$ as ``sample the parents and non-descendants of $v$, then sample $v$ conditional on the parents'' --- according to the definition, this is a possible choice for how to construct the observational distribution.

  In the second step, we are factoring the distribution on $\pa(v) \cup \nd(v)$ as ``sample the parents of $v$, then sample the non-descendants of $v$ \emph{according to the conditional distribution}''.
  This is always possible, and yields the diagram on the right.

  By~\citet[Remark 12.2]{fritz2019synthetic}, this implies the conditional independence
  $\nd(v) \bot v \mid \pa(v)$.
\end{proof}

\begin{remark}
  In fact, \cref{mech_is_cond,markov_prop} characterize the observational distribution uniquely.
  This can be proven diagrammatically by using a diagrammatic formulation of \cref{markov_prop} to show that the observational distribution factorizes as a certain diagram,
  and then using \cref{mech_is_cond} to show that the morphisms in this diagram may be replaced by the mechanisms.

  In classical terms, this argument corresponds to arguing that the probability factorizes according to the graph, and that the factors must be precisely the mechanisms.
\end{remark}

\clearpage
\section{Jensen-Shannon divergence is compositional}\label{sec:proofs}

\begingroup
\def\thetheorem{\ref{prop:short}}
\begin{proposition}[kernel composition is short]
  The composition of kernels
  \[\FinStoch(\mathbf{X},\mathbf{Y}) \otimes \FinStoch(\mathbf{Y},\mathbf{Z}) \to \FinStoch(\mathbf{X},\mathbf{Z})\]
   is a short map, that is,
  \[d_{\JSD}(f_{1}\circ g_{1},f_{2}\circ g_{2}) \leq d_{\JSD}(f_{1},f_{2}) + d_{\JSD}(g_{1},g_{2})\]
  for any $f_1,f_2\in\FinStoch(\mathbf{Y},\mathbf{Z})$,
  $g_1,g_2\in\FinStoch(\mathbf{X},\mathbf{Y})$.
\end{proposition}
\addtocounter{theorem}{-1}
\endgroup

\begin{proof}
  Since $d_{\JSD}$ is a metric, this is equivalent to the two statements
  \[d_{\JSD}(fg_{0},fg_{1}) \leq d_{\JSD}(g_{0},g_{1})\]
  \[d_{\JSD}(fg,f'g) \leq d_{\JSD}(f,f')\]

  In each case it suffices to show the given inequality at each $x \in \mathbf{X}$,
  so we can assume that $\mathbf{X} = *$. Since $x \mapsto \sqrt{x}$ is a monotone map,
  it suffices to show that
  \begin{align*}
  \JSD(fp_{0},fp_{1}) &\leq \JSD(p_{0},p_{1}) \\
  \text{and}\quad\JSD(f_{0}p,f_{1}p) &\leq \sup_{x}\JSD(f_{0}(x),f_{1}(x)),
  \end{align*}
  where $p_{0},p_{1}$ are distributions on $\mathbf{Y}$.

  The first case follows, since ``postprocessing'' the random variable $X$ can only reduce its mutual
  information with $B$.

  In the second case, we are comparing the following two situations:
  \begin{enumerate}
    \item If you learn the value of a random variable sampled from $f_{i}(x)$, how much do you learn about $i$, if $x$ is chosen to maximize this amount of information.
    \item If you learn the value of a random variable sampled from $f_{i}(x)$, where $x$ is chosen at random, how much do you learn about $i$?
  \end{enumerate}

  To see that the first number is bigger, consider the following:
  In the second case,
  even if you were additionally told what $x$ was (giving you \emph{more} information),
  you would still, at best, be in the first situation.
\end{proof}

\begingroup
\def\thetheorem{\ref{lemma:jsdcomposes}}
\begin{lemma}[JSD is compositional]
  Consider a diagram (not necessarily commutative)
    in $\FinStoch$ of the following form:
  \begin{equation}\label{commdiag}
  \begin{tikzcd}
    A \ar[colA,d, "a"] \ar[r, "f"] & A' \\
    B \ar[colA,r, "g"] \ar[colB,d, "b"] & B' \ar[colA,u, "b'"]\\
    C \ar[colB,r, "h"] & C' \ar[colB,u, "c'"]
  \end{tikzcd}
  \end{equation}
  Then $d_{\JSD}(f,{\color{colA}b'}{\color{colB}c'hb}{\color{colA}a})
  \leq d_{\JSD}(f, {\color{colA}b'ga}) + d_{\JSD}(g,{\color{colB}c'hb})$.
\end{lemma}
\addtocounter{theorem}{-1}
\endgroup

\begin{proof}
\begin{align*}
  d(f,{\color{colA}b'}{\color{colB}c'hb}{\color{colA}a}) &\leq d(f,{\color{colA}b'ga}) 
  + d({\color{colA}b'ga},{\color{colA}b'}{\color{colB}c'hb}{\color{colA}a}) \\
  &\leq d(f,{\color{colA}b'ga}) + d(g,{\color{colB}c'hb}),
\end{align*}
  where the first inequality is the triangle inequality and the last one uses \cref{prop:short}.
\end{proof}

\clearpage
\section{Enriched category theory and Error categories}\label{sec:enriched}

An \emph{enriched} category is a generalization of the concept of category, where the \emph{set} of maps \(x \to y\), \(\mathcal{C}(x,y)\), has been replaced by an object \(\mathcal{C}(x,y) \in \mathcal{V}\) in some other, \emph{enriching}, category.
For example, the set of linear maps \(V \to W\) between two vector spaces can itself be equipped with the structure of a vector space in
a canonical way – this defines an \emph{enrichment} of the category of vector spaces in itself.

A full discussion of enriched categories is beyond the scope of the present article; see \citet{kelly} for a comprehensive introduction.
The present paper contains two examples of enriched categories.

First, \cref{prop:short} shows that the Jensen-Shannon distance defines
an enrichment of \(\FinStoch\) in the category \(\Met\) of metric spaces:

\begin{definition}
  The category \(\Met\) of metric spaces is defined as follows:
  \begin{enumerate}
    \item The objects are metric spaces.
    \item A morphism $(X,d_{X}) \to (Y,d_{Y})$ is a function $f: X \to Y$ which is distance nonincreasing (or ``short''), meaning that $d_{Y}(f(x),f(x')) \leq d_{X}(x,x')$.
    \item Composition is function composition and the identity morphisms are the identity functions.
    \item The tensor product of two metric spaces is $(X,d_{X}) \otimes (Y,d_{Y}) = (X \times Y, d_{X} \otimes d_{Y})$, defined by $(d_{X}\otimes d_{Y})((x,y),(x',y')) = d_{X}(x,x') + d_{Y}(y,y')$.
  \end{enumerate}
\end{definition}
Then a category enriched in metric spaces consists of the following data:
\begin{enumerate}
  \item A category \(\mathcal{C}\)
  \item with a metric $d_{\mathcal{C}(X,Y)}$ on each set of morphisms $\mathcal{C}(X,Y)$ (for each pair of objects $X,Y \in \mathcal{C}$)
  \item so that, if $f,f': X \to Y, h: Y \to Z, g: A \to X$, we have \[d_{\mathcal{C}(X,Z)}(h\circ f,h\circ f'), d_{\mathcal{C}(A,Y)}(f \circ g, f' \circ g) \leq d_{\mathcal{C}(X,Y)}(f,f').\]
        In other words, post- and precomposition are distance nonincreasing maps.
\end{enumerate}

\cref{lemma:jsdcomposes} is essentially a lemma about enriched categories
– the statement and the proof make sense for any category enriched over \(\Met\).

Second, the compositional property of our error measure, \cref{prop:error_composes},
can be phrased to say that error defines an enrichment of \(\FinMod\)
in the following category of \emph{error spaces}.

\begin{definition}[Category of error spaces]
An \emph{error space} \((X,e)\) consists of a set \(X\) and a function \(e: X \to [0,\infty]\).
A morphism of error spaces \(f: (X,e_X) \to (Y,e_Y)\) is a function \(f: X \to Y\) so that \(e_Y(f(x)) \leq e_X(x)\) for all \(x \in X\).
The tensor product \((X,e_X) \tensor (Y,e_Y)\) of two error spaces is \((X \times Y, e_X \tensor e_Y)\), where \(e_X \tensor e_Y(x,y) = e_X(x) + e_Y(y)\).
This data defines a symmetric monoidal category \(\Err\) of error spaces.
\end{definition}

A category enriched in error spaces, or an \(\Err\)-category, then consists of the following data:

\begin{enumerate}
\item A category \(\mathcal{C}\),
\item with an error \(e(f)\) for each morphism \(f\) in \(\mathcal{C}\)
\item such that, when \(f,g\) are composable, \(e(fg) \leq e(f)+e(g)\).
\end{enumerate}

There is an error category of error spaces, \(\underline{\Err}\), where the maps \(f: (X,e_X) \to (Y,e_Y)\) are \emph{all} functions \(f: X \to Y\),
and the error of such a function is \(e(f) := \max(0,\sup_x e_Y(f(x)) - e_X(x))\);
that is, the error is the maximal increase in error created by the function.

In \cref{sec:functors}, we mentioned the notion of a \emph{functor}.
A functor is a mapping between categories preserving the compositional structure:
\begin{definition}[Functor]
  Given categories $\mathcal{C},\mathcal{D}$, a functor $F: \mathcal{C} \to \mathcal{D}$ consists of an object $F(X)$ in $\mathcal{D}$ for each object $X$ in $\mathcal{C}$,
  as well as a morphism $F(f): F(X) \to F(Y)$ for each $f: X \to Y$, so that $F(f\circ g) = F(f) \circ F(g)$ and $F(1_{X}) = 1_{F(X)}$.
\end{definition}

There is also a notion of functor between enriched categories.
Here, we spell this out for the case of \(\Err\)-categories:

A functor of \(\Err\)-categories (or \(\Err\)-functor) \(\mathcal{C} \to \mathcal{D}\)
then consists of the following data:

\begin{enumerate}
\item A functor \(F: \mathcal{C} \to \mathcal{D}\)
\item such that \(e(F(f)) \leq e(f)\).
\end{enumerate}

In particular, an \(\Err\)-functor \(F: \mathcal{C} \to \underline{\Err}\) consists of the following:
\begin{enumerate}
\item For each \(C \in \mathcal{C}\), an error space \(F(C)\),
\item for each map \(f: C \to D\), a function \(F(C) \to F(D)\)
\item such that \(F(f\circ g) = F(f) \circ F(g)\) and
\item such that \(e_{F(D)}(F(f)(x)) \leq e_{F(C)}(x) + e(f)\).
\end{enumerate}

Thus, this captures the desired properties of the collection of implemented models discussed in \cref{sec:functors}.

\bibliography{references}

\end{document}